\documentclass[sigconf, nonacm]{acmart}

\newcommand\vldbdoi{10.14778/3476249.3476295}
\newcommand\vldbpages{XXX-XXX}
\newcommand\vldbvolume{14}
\newcommand\vldbissue{11}
\newcommand\vldbyear{2021}
\newcommand\vldbavailabilityurl{http://vldb.org/pvldb/format_vol14.html}
\newcommand\vldbpagestyle{empty}

\usepackage{xspace}

\usepackage{subfigure}
\usepackage{multirow}
\usepackage{bm}
\usepackage{color, soul}
\usepackage{listings}
\usepackage{xcolor}
\usepackage[linesnumbered,vlined,ruled,noend]{algorithm2e}

\newcommand{\sys}{\textsc{Grain}\xspace}

\newcommand{\para}[1]{{\vspace{2pt} \bf \noindent #1 \hspace{1pt}}}
\definecolor{codegray}{rgb}{0.5,0.5,0.5}

\usepackage{enumitem}

\usepackage{wrapfig}

\begin{document}
\fancyhead{}
\title{\sys: Improving Data Efficiency of \emph{Gra}ph Neural Networks via Diversified \emph{In}fluence Maximization}
%\subtitle{Volcano Style Processing
%Meets AutoML [Scalable Data Science]}
% \subtitle{[Scalable Data Science]}

\renewcommand{\shorttitle}{\sys}
\author{Wentao Zhang$^{\dagger\ddagger}$, Zhi Yang$^{\dagger \S}$, Yexin Wang$^\dagger$, Yu Shen$^\dagger$, Yang Li$^\dagger$, Liang Wang$^{\dagger}$, Bin Cui$^{\dagger \S}$}
% \author{Yangyu Tao$^\ddagger$, Bin Cui$^\dagger$}
\affiliation{
{{$^\dagger$}{School of EECS \& Key Laboratory of High Confidence Software Technologies, Peking University}}~~~~~$^\ddagger$Tencent Inc.\\
 {{$^{\S}$}{Center for Data Science, Peking University \& National Engineering Laboratory for Big Data Analysis and Applications}}}
\affiliation{
$^\dagger$\{wentao.zhang, yangzhi, yexinwang, shenyu, liyang.cs, bin.cui\}@pku.edu.cn $^\dagger$ liang\_wang\_now@yahoo.com
}
\renewcommand{\shortauthors}{Zhang et al.}
%%
%% The "author" command and its associated commands are used to define the authors and their affiliations.

% \author{Yang Li$^{\dagger\ddagger}$, Yu Shen$^\dagger$, Wentao Zhang$^\dagger$, Jiawei Jiang$^\ddagger$, Bolin Ding$^\mathsection$, Yaliang Li$^\mathsection$}
% \author{Jingren Zhou$^\mathsection$, Zhi Yang$^\dagger$, Wentao Wu$^*$, Ce Zhang$^\ddagger$, Bin Cui$^\dagger$}
% \affiliation{
% $^\dagger$EECS, Peking University~~~~~$^\ddagger$ETH Zurich~~~~~$^\mathsection$Alibaba Group~~~~~$^*$Microsoft Research, Redmond
% }
% \affiliation{
% $^\dagger$\{liyang.cs, shenyu, wentao.zhang, yangzhi, bin.cui\}@pku.edu.cn $^\ddagger$\{jiawei.jiang, ce.zhang\}@inf.ethz.ch~~~~~\\$^\mathsection$\{bolin.ding, yaliang.li, jingren.zhou\}@alibaba-inc.com~~~~~$^*$wentao.wu@microsoft.com
% }

% \zwt{	1. 去掉number，底部
% 	2. 去除代码部分，tag
% 交表格，手写签字，VLDB}

\begin{abstract}
Data selection methods, such as active learning and core-set selection, are useful tools for improving the data efficiency of deep learning models on large-scale datasets. 
However, recent deep learning models have moved forward from independent and identically distributed data to graph-structured data, such as social networks, e-commerce user-item graphs, and knowledge graphs. 
This evolution has led to the emergence of Graph Neural Networks (GNNs) that
go beyond the models existing data selection methods are designed for. %whereas the existing data selection methods are designed for traditional deep learning models.
Therefore, we present \sys, an efficient framework that opens up a new perspective through connecting data selection in GNNs with \emph{social influence maximization}. 
By exploiting the common patterns of GNNs, \sys introduces a novel feature propagation concept, a diversified influence maximization objective with novel influence and diversity functions, and a greedy algorithm with an approximation guarantee into a unified framework. 
Empirical studies on public datasets demonstrate that \sys significantly improves both the performance and efficiency of data selection (including active learning and core-set selection) for GNNs. 
To the best of our knowledge, this is the first attempt to bridge two largely parallel threads of research, data selection, and social influence maximization, in the setting of GNNs, paving new ways for improving data efficiency.
\end{abstract}

\settopmatter{printfolios=false}
\maketitle
\pagestyle{\vldbpagestyle}
\begingroup\small\noindent\raggedright\textbf{PVLDB Reference Format:}\\
Wentao Zhang, Zhi Yang, Yexin Wang, Yu Shen, Yang Li, Liang Wang, Bin Cui. Grain: Improving Data Efficiency of Graph Neural Networks via Diversified Influence Maximization. PVLDB, \vldbvolume(\vldbissue): \vldbpages, \vldbyear.\\
\href{https://doi.org/\vldbdoi}{doi:\vldbdoi}
\endgroup
\begingroup
\renewcommand\thefootnote{}\footnote{\noindent
This work is licensed under the Creative Commons BY-NC-ND 4.0 International License. Visit \url{https://creativecommons.org/licenses/by-nc-nd/4.0/} to view a copy of this license. For any use beyond those covered by this license, obtain permission by emailing \href{mailto:info@vldb.org}{info@vldb.org}. Copyright is held by the owner/author(s). Publication rights licensed to the VLDB Endowment. \\
\raggedright Proceedings of the VLDB Endowment, Vol. \vldbvolume, No. \vldbissue\ %
ISSN 2150-8097. \\
\href{https://doi.org/\vldbdoi}{doi:\vldbdoi} \\
}\addtocounter{footnote}{-1}\endgroup

\ifdefempty{\vldbavailabilityurl}{}{
\vspace{.3cm}
\begingroup\small\noindent\raggedright\textbf{PVLDB Availability Tag:}\\
The source code of this research paper has been made publicly available at \textbf{\url{https://github.com/zwt233/Grain}}.
\endgroup
}

\section{Introduction}
Data selection methods, such as active learning and core-set selection, improve the data efficiency of deep learning on large datasets by identifying the most informative training examples.
In particular, active learning assists the learning procedure by prioritizing the selection of valuable unlabeled samples for human labeling, under the goal of maximizing the model performance with minimal labeling cost~\cite{settles2009active,Charu2014als,dasgupta2005analysis,li2013adaptive,Bligic2010alnd,Golovin2011asta}. 
Core-set selection techniques aim to find a small subset that accurately approximates the full dataset by selecting representative examples~\cite{kushal2007sck,huggins2016csblr,campbell2018bayesian,wei2014sssls,ni2015udsw,Wei2013udst}.

Graph Neural Networks (GNNs) have achieved state-of-the-art performance across various graph-based tasks such as node classification~\cite{zhang2020reliable,velivckovic2017graph,chen2018fastgcn,chiang2019cluster} and link prediction~\cite{he2020lightgcn,cui2020adaptive,wu2020garg,wu2020graph}.
However, GNNs require plenty of labeled data to achieve satisfactory performance, and substantial training times. %training GNNs is time-consuming.
Therefore, introducing data selection methods for GNNs is crucial.
Unfortunately, existing data selection methods fail to address the following new challenges posed by GNNs in terms of both performance and efficiency:
%are often ineffective when applied to GNNs , \blue{due }.
%This is because GNNs pose new challenges to data selection:

First, GNNs are graph-based semi-supervised learning models, which aggregate the feature of a node with its nearby neighbors. 
However, most existing data selection methods are designed to learn models on independent and identically distributed (i.i.d) data, which fail to model the interactions imposed by the graph structure. 
Several attempts~\cite{prasad2014submodular,long2008graph} have been made for data selection on graphs, but they are ineffective to GNNs since they cannot capture both graph structure and node features.
As a result, it is necessary to design a new data selection criterion that is in coherence with the characteristics of GNNs to select the most valuable samples. 

Second, GNNs incur substantial training costs and are hard to scale to large graphs because they need to perform a recursive neighborhood expansion to compute the hidden representations of a given node. 
Most existing active learning techniques select a batch of samples with guidance from the previously trained model and have 
to retrain a computationally-expensive model once a new labeled example comes.
So it is expensive to directly apply them to GNNs, which hampers their applicability in real-world applications.

In this paper, we propose \sys, a novel data selection framework towards efficient data selection for GNNs.
The working philosophy of \sys is to (1) take the essential component of GNNs -- feature propagation -- as a type of influence propagation, and (2) maximize the efficiency of feature propagation with analogy to \emph{social influence maximization}~\cite{krause2008near}. 
Concretely, we measure the sensitivity of node $v$ to node $u$ (i.e., the ``influence'' of $u$ on $v$) by computing how much the input feature of $u$ affects the aggregated feature of $v$ through feature propagation. 
Then we \emph{maximize the influence} of the selected nodes (i.e., seed nodes) to get more unlabeled nodes influenced and involved in the downstream model training.
In contrast to the traditional methods that consider informative scores of individual nodes, \sys is more effective in exploiting the interaction among nodes, by \emph{explicitly} maximizing the number of unlabeled nodes influenced by labeled ones.

Although the definition of feature influence is inspired by ~\cite{wang2020unifying,DBLP:conf/icml/XuLTSKJ18}, the non-obvious observation that data selection for GNNs can be modeled as influence maximization is a major contribution. Moreover, we are making contributions in bridging these two
largely parallel threads of research, active learning
and social influence maximization, in new GNN settings:
First, unlike the classic linear threshold (LT) model and independent cascade (IC) model~\cite{kempe2003maximizing} which use social influence maximization, we define a novel feature influence-based propagation model.
Besides considering the direct influence via feature propagation, we further recognize the \emph{indirect influence} of GNNs over feature space: nodes that are close in feature space are likely to have the same label. 
Thus, when a node is influenced via direct feature propagation, the propagated influence/label signal can also be viewed as indirectly interpolating a smooth signal~\cite{li2018deeper} to nodes with close feature distance from the influenced nodes in feature space. 
To further improve the effectiveness of feature influence, we incorporate the diversity of influenced crowds to enforce the indirect influence, which encourages the influenced nodes to cover more semantic categories. 

Based on the above perspective, we propose a novel data selection criterion for GNNs by unifying both the magnitude and the diversity of the influenced crowd into a \emph{Diversified Influence Maximization} (DIM) criterion. %problem 
For effective influence maximization, we construct new influence and diversity functions in \sys framework.
We also prove that the proposed diversity functions and the objective functions have good properties of submodularity and monotonicity, which enable a simple greedy algorithm to find a near-optimal result.
Meanwhile, \sys has the advantages of high efficiency and scalability over the existing methods, as it reduces the training cost by separating feature propagation from the training of GNNs. 

%\sys further provides relaxations for extremely high efficiency 
%based on the characteristics of real-world graphs such as the power-law degree distributions and homophily phenomenon.
 
%The key novelty of \sys is that it bridges two largely parallel threads of research, data selection and social influence maximization in the setting of GNNs, leading to the following contributions:
In summary, the core contributions of this paper are:
(1) {\textbf{Problem Connection.}} We open up a novel perspective for improving the data efficiency of GNNs by connecting GNN active learning with social influence maximization.
(2) {\textbf{New Criterion.}} We propose a fundamentally new data selection criterion--\emph{diversified influence maximization}--for GNNs by considering both the magnitude and diversity of feature influence.
Moreover, we propose a greedy algorithm to solve this problem with an approximation guarantee.
(3) {\textbf{Novel Functions.}} 
We propose novel influence and diversity functions based on both direct and indirect influence of GNNs over the graph and feature space, respectively, which have good properties of submodularity and monotonicity.
(4) \textbf{High Performance and Efficiency.}
Through experiments on real-world graphs with typical GNNs, we demonstrate that \sys significantly outperforms the state-of-the-art baselines on both performance and efficiency.

\section{Preliminary}
In this section, we first describe the notations and define two data selection problems: Active Learning and Core-set Selection. Then we introduce GNNs and social influence maximization.

\subsection{Data Selection Problems}
We are given a graph $\mathcal{G}=(\mathcal{V},\mathcal{E})$ with $|\mathcal{E}|=N$ nodes, and its adjacent matrix $\mathbf{A}\in\mathbb{R}^{N\times N}$. 
Each node $v_i \in \mathcal{V}$ is associated with a feature vector $\boldsymbol{x}_i \in \mathbb{R}^d$, which forms the feature matrix $\mathbf{X} \in \mathbb{R}^{N\times d}$. 
The ground-truth label for node $v_i$ is a one-hot vector $\boldsymbol{y}_i \in \mathbb{R}^{C}$ where $C$ is the number of classes and the $c$-th element is 1 only if node $v_i$ belongs to class $c$. 
The entire node set $\mathcal{V}$ is partitioned into the training set $\mathcal{V}_{train}$, 
validation set $\mathcal{V}_{val}$ and test set $\mathcal{V}_{test}$. 
The training algorithm is denoted as $M$, which is GNN in this paper.

\textit{\underline{Active Learning.}}
Given the unlabeled node set $\mathcal{V}_{train}$ and a loss function $\ell$ (e.g., generalization error), the goal of active learning is to select a subset of $\mathcal{B}$ nodes $S$ from $\mathcal{V}_{train}$ to label so that the lowest loss on the test set $\mathcal{V}_{test}$ can be achieved after applying $S$ to algorithm $M$. 
The formulation of active learning is as follows:
\begin{small}
\begin{equation}
\mathop{\arg\min}_{S:|S|=\mathcal{B}}\mathbb{E}_{v_i\in \mathcal{V}_{test}}\left[\ell\left(\boldsymbol{y}_i, P(\hat{\boldsymbol{y}}_i|\boldsymbol{x}_i;M_S)\right)\right],
\label{eq:target}
\end{equation}
\end{small}
where $P(\hat{\boldsymbol{y}}_i|\boldsymbol{x}_i;M_S)$ is the label distribution of node $v_i$ predicted by $M_S$, and $M_S$ is trained under the supervision of the labeled set $S$.

Recent researches~\cite{gal2017deep,sener2018active,kirsch2019batchbald,zhang2021alg} focus on using deep learning models as the learning algorithm $M$ and solve the problem of active learning in the batch setting.
During $b$ rounds, they select $\frac{\mathcal{B}}{b}$ data points to label in each round. We denote the selected points during the $i$-th round as $s_i$. At the $k$-th round, these methods refit the learning algorithm $M$ on all the labeled data $\bigcup_{i \in [k]}s_i$ to avoid any correlation between selections~\cite{kirsch2019batchbald, Snapshot_boosting}. Specifically, Wolf et. al.~\cite{wolf2011facility}, and Sener et. al.~\cite{sener2018active} propose the greedy k-centers method to select data points that maximize the minimal distance between the current point and the labeled ones. Settles~\cite{settles2012active}, Shen et. al.~\cite{shen2017deep}, and Gal et. al.~\cite{gal2017deep} select points based on predictive confidence, i.e., the highest probability belonging to a certain class predicted by $M$. 

While the above methods concentrate on supervised learning, AGE~\cite{cai2017active} and ANRMAB~\cite{gao2018active} are proposed for active learning on semi-supervised scenarios (e.g., graph) by using both the node features and graph structure.
Recently, clustering-based active learning methods have been proposed for GNNs, such as Featprop~\cite{wu2019active} and LSCALE~\cite{liu2020active}. Featprop uses propagated node features as representations to cluster unlabeled nodes and labels the cluster centers. 
LSCALE proposes a learned latent space for clustering, which combines the unsupervised learning features and supervised hidden representations to select nodes for labeling. 
Compared with the related work, \sys fully utilizes the
hidden information of unlabeled nodes, by \emph{explicitly} maximizing the number of unlabeled nodes influenced by labeled ones over the graph and feature space (i.e., getting more unlabeled nodes involved in GNN training).
 
\textit{\underline{Core-set Selection.}}
Suppose that $\mathcal{V}_{train}$ is full labeled, core-set selection can be defined as methods that search for a subset of data points that maintain a similar level of quality (e.g., classification error) with the entire training set $\mathcal{V}_{train}$. Concretely, the goal of core-set selection is to select a subset of $\mathcal{B}$ nodes $S \subset \mathcal{V}_{train}$ that achieves comparable result as the whole training set $\mathcal{V}_{train}$, that is:

\begin{footnotesize}
\begin{equation}
\mathop{\arg\min}_{S:|S|=\mathcal{B}} \mathbb{E}_{v_i\in \mathcal{V}_{test}} \left[\ell\left(\boldsymbol{y}_i, P(\hat{\boldsymbol{y}}_i|\boldsymbol{x}_i;M_S)\right) - \ell\left(\boldsymbol{y}_i, P(\hat{\boldsymbol{y}}_i|\boldsymbol{x}_i;M_{\mathcal{V}_{train}})\right)\right],
\end{equation}
\end{footnotesize}
where $M_S$ and $M_{\mathcal{V}_{train}}$ are algorithms trained under the supervision of the labeled set $S$ and the full training set $\mathcal{V}_{train}$ respectively.

Besides the greedy k-centers method~\cite{wolf2011facility,sener2018active} mentioned above, we introduce another two core-set selection techniques in previous work: forgetting events, and max entropy.
Toneva et. al.~\cite{toneva2018empirical} define forget events as the number of times an example is incorrectly classified after having been correctly classified earlier during training and select the points with the highest number of forgetting events. 
Lewis et. al.~\cite{lewis1994sequential} and Settles~\cite{settles2012active} rank the entropy of the predictions from algorithm $M$ and keep the points with the highest entropy.

\subsection{Graph Neural Networks}

Graph Neural Networks (GNNs)~\cite{DBLP:conf/iclr/KipfW17,hamilton2017inductive,velivckovic2017graph} define a multi-layer message passing process, through which the feature representation of a node in the next layer could be the aggregation of its neighborhood in the current layer. They differ in the way of defining the recursive function $f$ for message passing:
\begin{small}
\begin{equation}
   \mathbf{X}^{(k)} \gets f(\mathbf{X}^{(k-1)}, \mathbf{A}, \mathbf{\Theta}^{(k)}),
\end{equation}
\end{small}
where $\mathbf{X}^{(k)}$ and $\mathbf{\Theta}^{(k)}$ are the output embedding and trainable parameters of layer $k$. Naturally, the input $\mathbf{X}^{(0)}$ satisfies $\mathbf{X}^{(0)}=\mathbf{X}$.

For example, Graph
Convolution Network (GCN)~\cite{DBLP:conf/iclr/KipfW17} has a specific form of the function $f$ as:
\begin{small}
\begin{equation}
    \mathbf{X}^{(k)} = \delta \left(\widetilde{\mathbf{D}}^{-\frac{1}{2}}\widetilde{\mathbf{A}}\widetilde{\mathbf{D}}^{-\frac{1}{2}}\mathbf{X}^{(k-1)}\mathbf{\Theta}^{(k)}\right),
\end{equation}
\end{small}
where $\widetilde{\mathbf{D}}$ is the diagonal degree matrix of $\widetilde{\mathbf{A}}=\mathbf{A}+\mathbf{I}_{N}$, and $\mathbf{I}_{N}$ is the identity matrix.
$\mathbf{\Theta}^{(k)}$ is the layer-specific trainable weight matrix and $\delta(\cdot)$ is the non-linear activation function. 
For classification tasks, a K-layer GCN applies a softmax function on the aggregated representation in the final layer to obtain the prediction score for each class. That is,
$
\hat{\mathbf{Y}}=\text{softmax}
\left(\widetilde{\mathbf{D}}^{-\frac{1}{2}}\widetilde{\mathbf{A}}\widetilde{\mathbf{D}}^{-\frac{1}{2}}\mathbf{X}^{(K-1)}\mathbf{\Theta}^{(K)} \right).
$

\subsection{Social Influence Maximization}
The influence maximization (IM) problem in social networks aims to select $\mathcal{B}$ nodes so that the number of nodes activated (or influenced) in the social networks is maximized~\cite{kempe2003maximizing}. 
Namely, given a graph $\mathcal{G} = (\mathcal{V}, \mathcal{E})$, the formulation is as follows:
\begin{equation}
\begin{aligned}
\max_S\ {|\sigma(S)}|,\ \mathbf{s.t.}\ {S}\subseteq \mathcal{V},\ |{S}|=\mathcal{B},
\end{aligned}
\label{max_thetaS}
\end{equation}
where $\sigma(S)$ is the set of nodes activated by the seed set $S$ under certain influence propagation models, such as Linear Threshold (LT) and Independent Cascade (IC) models~\cite{kempe2003maximizing}.
The maximization of $\sigma(S)$ is NP-hard. However, if $\sigma(S)$ is nondecreasing and submodular with respect to $S$, a greedy algorithm can provide an approximation guarantee of $(1 - \frac{1}{e})$~\cite{nemhauser1978analysis}.
Great efforts have been devoted to the IM problem ~\cite{jung2012irie,yang2012approximation,chen2013information, zhao2021maximizing}, and some researches also explore the diversity over the activated nodes in social influence maximization~\cite{krause2008near,lin2011class, dey2013contextual}. But they are designed for specific models or tasks and cannot be applied to GNNs directly. 
%\sys is the first work that connects social influence maximization with data selection in GNNs by defining novel feature influence model and diversity functions. 

This paper focuses on \emph{bridging} data selection
and diversified influence maximization (DIM) in new GNN settings, rather than addressing classic social influence maximization problems in previous work. We propose a novel propagation model and diversity functions by exploiting the characteristics of GNNs, including both direct and indirect feature influence. 
We further propose a fundamentally new selection criterion by connecting GNN data selection with DIM, and demonstrate that it outperforms traditional learning-based AL approaches (i.e., achieves state-of-the-art performance for GNN data selection problem). 
\sys leverages the literature dedicated to (diversified) social influence maximization and exhibits both the feasibility and potential of such connection, which opens up a promising future direction in GNN-based active learning.

\section{\sys{} Framework}
In this section, we present \sys, a new data selection framework for GNNs.
As illustrated in Figure~\ref{fig:framework}, at each round of data selection, \sys takes as input a graph $\mathcal{G}=\mathcal{(V,E)}$, feature matrix $\mathbf{X}$, and computes the proposed influence and diversity measurements based on the direct and indirect influence over graph structure and $k$-steps aggregated embedding $\mathbf{X}^{(k)}$, respectively.
Next, \sys combines the influence and diversity into a unified criterion — maximizing the diversified influence and selects a node based on this new criterion. 
The procedure is repeated until the labeling budget $\mathcal{B}$ exhausts.
Below, we introduce each component of \sys in detail.

\subsection{Feature Influence Model}
\textit{\underline{Decoupled Feature Propagation.}} Graph neural networks define a multi-layer feature propagation process. 
While feature propagation and transformation might be intertwined in many GNNs,
recent studies have observed that GNNs primarily derive their benefits from performing feature
smoothing over graph neighborhoods rather than learning non-linear hierarchies of features as implied by the analogy to CNNs~\cite{sign_icml_grl2020,he2020lightgcn,wu2019simplifying}.
Thus, we separate the essential operation of
GNNs — feature propagation — inherited from GNNs by removing the neural network $\mathbf{\Theta}$ and
non-linear activation for feature transformation. 
Specifically, we construct a parameter-free $K$-step feature propagation process for a target $K$-layer GNN:
\begin{small}
\begin{equation}
    \mathbf{X}^{(k)} \gets f(\mathbf{X}^{(k-1)}, \mathbf{T}, \mathbf{X}^{(0)}),\ \forall k=1,\ldots, K. 
    \label{equ:prop}
\end{equation}
\end{small}
where $\mathbf{T}$ is the generalized transition matrix
% Examples for $\mathbf{T}$ in an undirected graph include the random walk transition matrix
%  $\mathbf{T}_{rw}= \tilde{\mathbf{D}}^{-1}\tilde{\mathbf{A}}$
% , the symmetric transition matrix $\mathbf{T}_{sym} =\tilde{\mathbf{D}}^{-\frac{1}{2}}\tilde{\mathbf{A}}\tilde{\mathbf{D}}^{-\frac{1}{2}}$, and triangle. IA matrix $\mathbf{T}_{tr} = \mathbf{D}_T^{-1}\mathbf{A}_T$.
, e.g., the random walk transition matrix
 $\mathbf{T}_{rw}= \tilde{\mathbf{D}}^{-1}\tilde{\mathbf{A}}$
, the symmetric transition matrix $\mathbf{T}_{sym} =\tilde{\mathbf{D}}^{-\frac{1}{2}}\tilde{\mathbf{A}}\tilde{\mathbf{D}}^{-\frac{1}{2}}$, the triangle. IA matrix $\mathbf{T}_{tr} = \mathbf{D}_T^{-1}\mathbf{A}_T$, etc.

\begin{figure}[t]
\centering
\includegraphics[width=.98\linewidth]{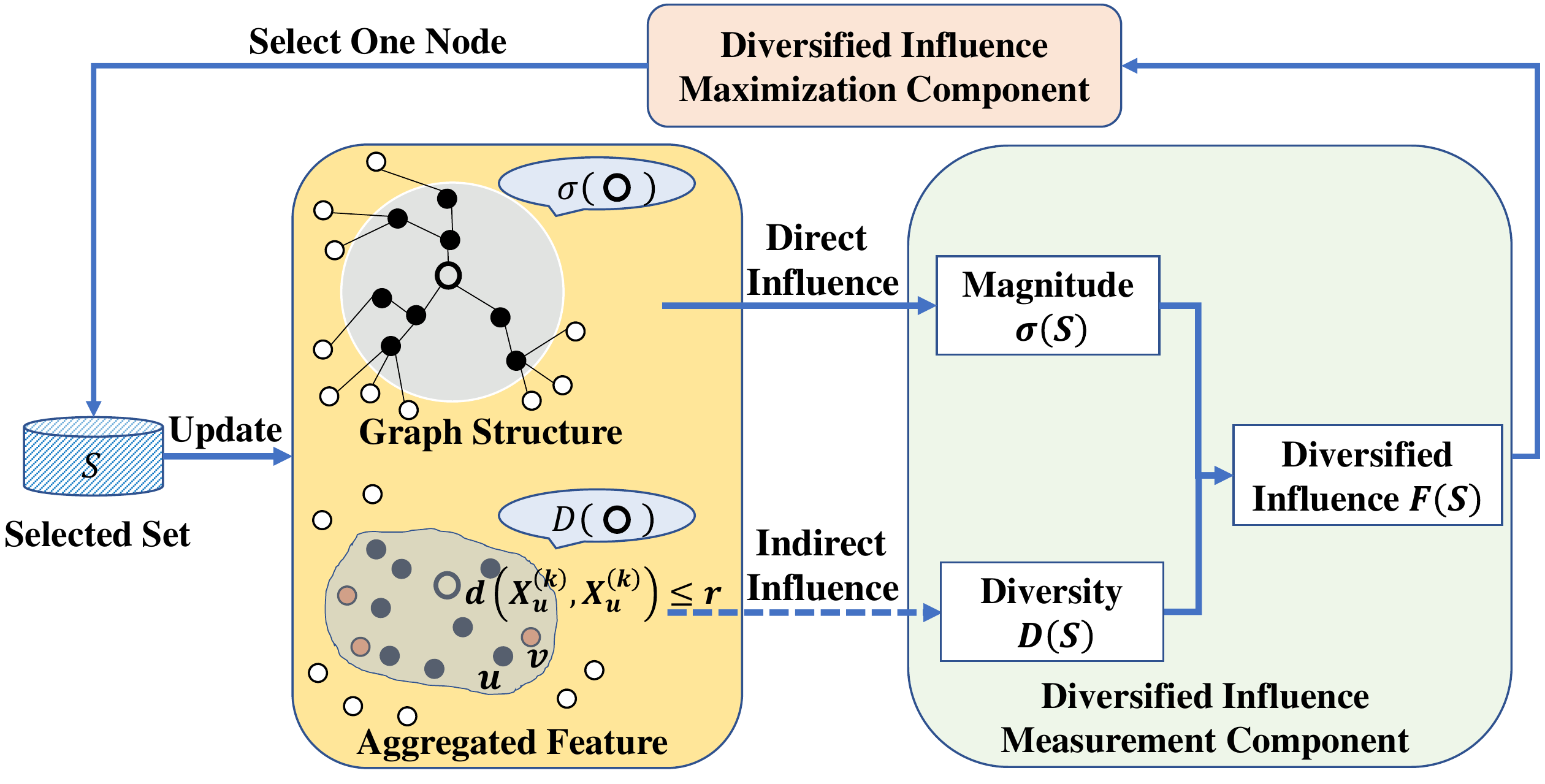}
% \vspace{-3mm}
\caption{The workflow of \sys.}
\label{fig:framework}
\end{figure}

\textit{\underline{Feature Influence Viewpoint.}} Under the given decoupled propagation mechanism, $\mathbf{X}^{(k)}$ is the aggregated feature obtained by propagating features from nodes within
$k$-hop neighborhood. Hence, $\mathbf{X}^{(k)}$ captures the information from the subtree of
height $k$ rooted at individual nodes. 
By taking feature propagation as a type of influence propagation, we open up a new perspective for the data selection problems in GNNs from the
viewpoint of influence maximization (IM). Inspired by~\cite{wang2020unifying,DBLP:conf/icml/XuLTSKJ18}, we measure the feature influence of a node $u$ on $v$ by
by how much a change in the input feature of $u$ affects the \emph{aggregated} feature of $v$ after $k$ iterations.
\begin{definition}[Feature Influence] The feature influence score of
node $u$ on node $v$ after k-step propagation is the L1-norm of
the expected Jacobian matrix:
\begin{small}
\begin{equation}
    I(v,u,k)=\left\lVert \mathbb{E}[\partial \mathbf{X
    }_v^{(k)}/\partial \mathbf{X
    }_u^{(0)}]\right\rVert_1.
\end{equation}
\end{small}
The normalized influence score is defined as:
\begin{small}
\begin{equation}
I_v(u,k)=\frac{I(v,u,k)}{\sum_{w\in \mathcal{V}}I(v,w,k)}.\label{eq:inf}
\end{equation}
\end{small}
\end{definition}

Given the $k$-step feature propagation mechanism~\eqref{equ:prop}, the feature influence score $I_v(u,k)$ captures the sum
over probabilities of all possible influential paths %of length $k$ 
from $v$ to $u$. For example, let the generalized transition matrix be $\mathbf{T}=\mathbf{T}_{rw}$, the $I_v(u,k)$ is the probability that a random walk starting at $v$ ends at $u$ after taking $k$ steps: 
\begin{small}
\begin{equation}
    I_v(u,k)=\sum_{\mathcal{P}^{v\rightarrow u}_k}\prod_{i=k}^{1}\widetilde{a}_{v_{(i-1)},v_{(i)}},
    \label{eq:walkdis}
\end{equation}
\end{small}
where $\mathcal{P}^{v\rightarrow u}_k$ 
be a path $[v_{(k)},
v_{(k-1)}
, \ldots, 
v_{(0)}]$ of
length $k$ from node $v$ to $u$ and $\widetilde{a}_{v_{(i-1)},v_{(i)}}$ is the normalized weight of edge $(v_{(i)},v_{(i-1)})$. 

%Interestingly,
%many label propagation algorithms adopt random walks to propagate labels. Thus, we could also take $I_v(u,k)$ as the probability that \underline{$u$ propagates its label to $v$} via random walks from the label influence perspective. 

\begin{small}
\begin{table}[t]
    \caption{Propagation mechanisms adopted in various GNNs. Note that each edge weight in $A_T$ means the number of different triangles it belongs to and Triangle. IA is the abbreviation of Triangle-induced Adjacency.}
    \centering
    \resizebox{0.95\linewidth}{!}{
    \begin{tabular}{l|l}
        \toprule
        \textbf{Prorogation Mechanism} & \textbf{Formula}\\
        \midrule
         Normalized Adjacency~\cite{DBLP:conf/iclr/KipfW17} & $\mathbf{X}^{(k)}=\mathbf{T}_{sym}\mathbf{X}^{(k-1)}%\tilde{\bm{D}}^{-\frac{1}{2}}\tilde{\bm{A}}\tilde{\bm{D}}^{-\frac{1}{2}}\mathbf{X}^{(k-1)}
         $\\
         Random Walk~\cite{wu2019simplifying} & $\mathbf{X}^{(k)}=\mathbf{T}_{rw}\mathbf{X}^{(k-1)}
         %\tilde{\bm{D}}^{-1}\tilde{\bm{A}}\mathbf{X}^{(k-1)}
         $\\
         PPR~\cite{DBLP:journals/corr/abs-1810-05997} & $\mathbf{X}^{(k)}=(1-\alpha)\mathbf{T}_{rw}\mathbf{X}^{(k-1)}+\alpha \mathbf{X}^{(0)}$ %\alpha\in \left(0,1 \right]
         %(1-\alpha)\tilde{\bm{D}}^{-\frac{1}{2}}\tilde{\bm{A}}\tilde{\bm{D}}^{-\frac{1}{2}}\mathbf{X}^{(k-1)}+\alpha \mathbf{X}^{(0)}$ %\alpha\in \left(0,1 \right]
         \\
        Triangle. IA~\cite{sign_icml_grl2020} & $\mathbf{X}^{(k)}= \mathbf{T}_{tr}\mathbf{X}^{(k-1)}
        %\bm{D}_T^{-1}\bm{A}_T\mathbf{X}^{(k-1)}
        $\\
        S$^2$GC~\cite{zhu2021simple} & $\mathbf{X}^{(k)}=
        \frac{1}{k}((1-\alpha)T^k\mathbf{X}^{(0)} + \alpha \mathbf{X}^{(0)} + (k-1)\mathbf{X}^{(k-1)})
        %\sum_{l=1}^{k} (1-\alpha)(\tilde{\bm{D}}^{-\frac{1}{2}}\tilde{\bm{A}}\tilde{\bm{D}}^{-\frac{1}{2}})^l\mathbf{X}^{(0)}+\alpha \mathbf{X}^{(0)}
        $\\
        GBP~\cite{DBLP:conf/nips/ChenWDL00W20} & $\mathbf{X}^{(k)}= \theta_k T^k \mathbf{X}^{(0)} +  \mathbf{X}^{(k-1)}
        %\sum_{l=0}^{k} \alpha_l(\tilde{\bm{D}}^{-\frac{1}{2}}\tilde{\bm{A}}\tilde{\bm{D}}^{-\frac{1}{2}})^l\mathbf{X}^{(0)}
        $\\
        \bottomrule
    \end{tabular}}
    \label{tlb:prop}
\end{table}
\end{small}

\textit{\underline{Discussions.}}GNN model family adopts a variety of propagation mechanisms, such as Random Walk and Personalized PageRank (PPR), as summarized in Table~\ref{tlb:prop}. 
The performance of different mechanisms depends on the task, graph structure, and features jointly. For example, in social networks, triangle-based propagation might help distinguishing edges representing weak or strong ties.
On graphs with noisy connectivity, PPR may work well.
One could also add other propagation kernels following Eq.~\eqref{equ:prop}. \sys can be applied to a large variety of GNNs by i) adopting similar propagation kernels used by these GNNs and ii) computing node influence with Eq.~\eqref{eq:inf}. In this way, we offer a general approach for GNN active learning from the novel perspective of feature influence.

%Since we are focus on data selection for the target GNN models, we just select the mechanism used by the target model.

The decoupled feature propagation in Eq.~\eqref{equ:prop} enables \sys to perform node selection in a model-free manner, i.e., it selects the nodes to label once and for all before the GNN model starts. 
It is worth pointing that such decoupling in \sys \emph{does not strictly} require the decoupled feature propagation in downstream GNN models. 
Besides, \sys also supports GNN with interwined feature propagation and DNN transformation (e.g., GCN) and self-supervised GNN model (e.g., MVGRL~\cite{hassani2020contrastive}). 
The main reason for this generality is that: for both coupled and decoupled GNNs, (1) they primarily benefit from performing feature propagation over graph rather than learning non-linear hierarchies of features~\cite{wu2019simplifying}, and (2) their node influence distributions are consistently connected to probabilities of possible influential paths, e.g.,  \cite{DBLP:conf/icml/XuLTSKJ18} proves that the influence distribution in GCN is equivalent to the random walk distribution given by Eq.~\eqref{eq:walkdis} when choosing $\mathbf{T}=\mathbf{T}_{rw}$ in Eq.~\eqref{equ:prop}.

%\subsection{Influence Function}
\subsection{Diversified Influence Maximization}
\textit{\underline{Influence Function.}} Intuitively, the weak influence of a label node $u$ on $v$ with small probability $I_v(u,k)$ would have limited impact on $v$ due to few influence paths to propagate labels. 

Therefore, we
define the activation of a node $v$ by requiring the maximum
influence $I_v(S,k) = \max_{u\in S} I_v(u,k)$ of $S$ on the node $v$ is larger than a certain threshold value.

\begin{definition}[Activated Node Set]
Given a threshold $\theta$, $k$-step feature propagation~\eqref{equ:prop}, and a set of seeds $S$ , the activated node set $\sigma(S)$ is a subset of nodes in $\mathcal{V}$ that can be
activated by $S$:
\begin{small}
\begin{equation}
\sigma(S) = \mathop{\bigcup}_{v\in \mathcal{V}, I_v(S,k) > \theta }\{v\},
\label{mag}
\end{equation}
\end{small}
\end{definition}

Note that the threshold $\theta$ controls the value at which a node $v$ is activated (or significantly influenced) by the labeled nodes set $S$. 
An inactive node $v$ becomes active if $I_v(S,k) > \theta$.
Therefore, a larger $\theta$ requires the labeled node set $S$ to impose a stronger influence on the node $v$ in order to activate it. The budget $\mathcal{B}$ here is the number of nodes to label, i.e., $|S|\leq \mathcal{B}$. 
Given a small budget, the overall influence of the labeled node set $S$ on other nodes is relatively weak. In this case, we set a small threshold $\theta \rightarrow 0$ to make nodes more easily to be activated, so that we can choose nodes (to label) that can activate more unlabeled nodes for model training. Otherwise, few nodes can be activated by $S$. On the contrary, given enough labeling budget, the overall influence of the label node set $S$ is relatively larger, and we should set a larger threshold $\theta > 0$ to ensure sufficient influence of $S$ on each activated node $v$.

%The threshold $\theta=0$ means we consider an unlabeled
%node $v$ is influenced as long as there is a $k$-step path from any labeled node, which is equal to measure whether this unlabeled node
%can be involved in the full training process of GNNs. In practice, we
%could choose this threshold value in the case of a very small budget,
%so that our goal is to involve unlabeled nodes as many as possible.
%However, if the budget is relatively large, we could choose a positive
%threshold $\theta>0$, which enables the selection process to pay more
%attention on those weakly influenced nodes.

\begin{figure}[tp]
% \vspace{-4mm}
\centering  
\subfigure[Influence magnitude $|\sigma(S)|$]{
\label{fig:inf}
\includegraphics[width=0.23\textwidth]{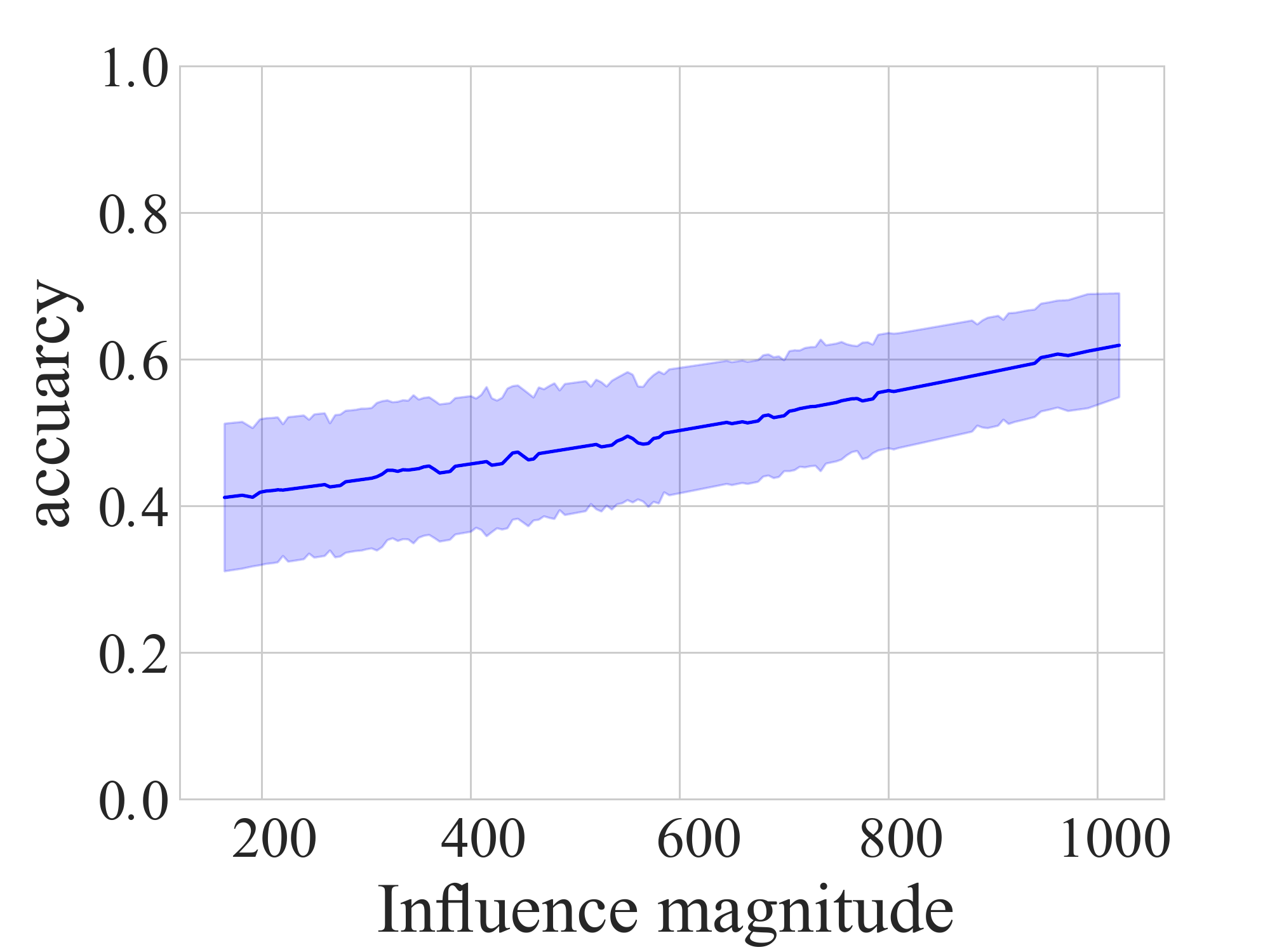}}
\subfigure[Influence diversity $D(S)$]{
\label{fig:div}
\includegraphics[width=0.23\textwidth]{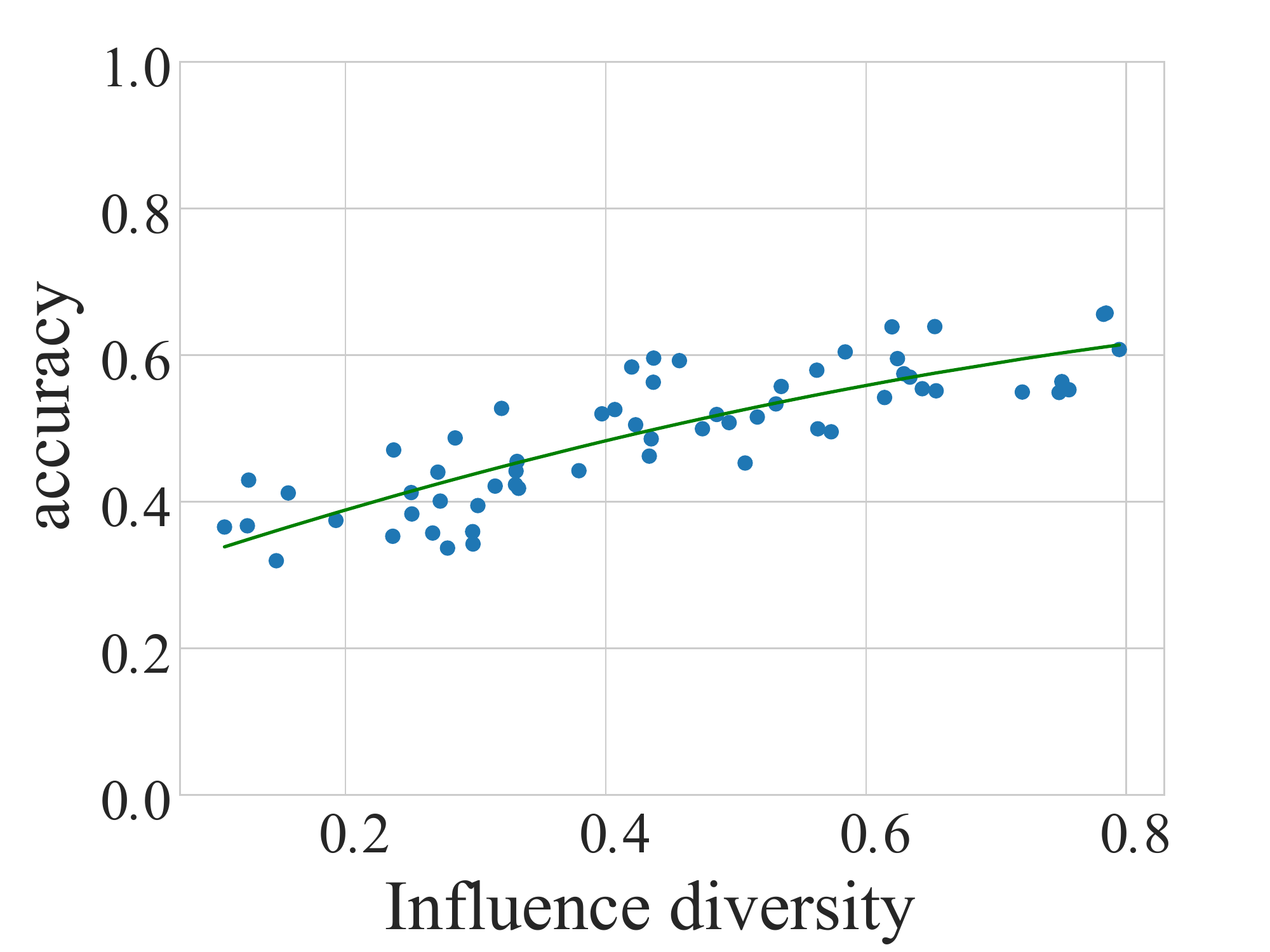}}
% \vspace{-4.5mm}
\caption{The relationship between seed nodes $S\ (|S|=20)$ and the test accuracy of GCN model trained on $S$ of Cora.}
\label{rp_all}
% \vspace{-4mm}
\end{figure}

Since increasing the influence of nodes that are already activated cannot benefit $\sigma(S)$, it is easy to derive the following theorem:
\begin{theorem}
$|\sigma(S)|$ is nondecreasing and submodular with respect to $S$, i.e., $\forall S\subseteq T, v\notin T, |\sigma(T)| \geq |\sigma(S)|$ and $|\sigma(S\cup\{v\})|-|\sigma(S)|\geq |\sigma(T\cup\{v\})|-|\sigma(T)|$.
\end{theorem}

\textit{\underline{Influence Maximization.}}
To increase the feature/label influence (smoothness) effect on graphs, we should select nodes that can influence more unlabeled nodes. 
Due to the impact of graph structure, the speed of expansion, i.e., the growth of the influence, can change dramatically given different sets of label nodes. 
This observation motivates us to address the graph data selection problem in the viewpoint of influence maximization defined in Eq.~\eqref{max_thetaS}.
To illustrate this insight, we randomly select different sets of $\mathcal{B} = 20$ labeled nodes and train a GCN model with the different labeled set on Cora. We sort these label sets in terms of $\sigma(S)$, and Figure~\ref{fig:inf} plots the relationships between influence magnitude $\sigma(S)$ and the GCN model accuracy trained with the label supervision of $S$. Even with the same number of 20 labeled nodes, Figure~\ref{fig:inf} shows that the accuracy tends to increase along with their influence magnitude, implying the potential gain of exploring the influence viewpoint.

% \vspace{-1mm}
\textit{\underline{Introducing Diversity.}}
From Figure~\ref{fig:inf}, we also observe that the node set $\sigma(S)$ with a similar influence still leads to different accuracy (i.e., high variance).  
This is because simply maximizing $\sigma(S)$ fails to model the interactions between nodes, i.e., activating one node could often affect the utility of activating another one. 
Since the aggregated features $\mathbf{X}^{(k)}$
contains both the feature and graph structure information,
it is reasonable to assume that nodes that are close together in the aggregated feature space will be likely to have the same label.
So when a node $v\in \sigma(S)$ is activated over the graph structure, the propagated influence/label signal can also be viewed as \emph{indirectly} interpolating
a smooth signal to nodes with close feature distance from $v$ at the feature space, which we refer to as \emph{indirect influence}.
So besides the direct influence by feature propagation over the graph structure, we also consider enforcing this indirect influence over feature space to further improve the overall effectiveness of influence. 
To this end, the diversity of activated nodes
$\sigma(S)$ needs to be explored in the feature space. Specifically, we expect more nodes in $\sigma(S)$ can scatter over different regions of $\mathcal{V}$ in the feature space, aim to allow each node in $\mathcal{V}$ can find a close node belonging to $\sigma(S)$. We shall introduce our diversity functions in the later Section~\ref{sec:diveristy}.

% \vspace{-1mm}
\textit{\underline{Proposed Criterion.}} 
In this paper, we propose a new GNN data selection criterion to consider the magnitude of influence and the diversity of the influence simultaneously, from a novel viewpoint of diversified influence maximization.
Specifically, our \sys framework adopts a diversified influence
maximization objective:
\begin{small}
\begin{equation}
\max_S\ F(S)= \frac{{|\sigma(S)|}}{\hat{{|\sigma|}}} + \gamma \frac{{D(S)}}{\hat{D}},\  \mathbf{s.t.}\ {S}\subseteq \mathcal{V},\ |{S}|=\mathcal{B}.
\label{eq:obj}
\end{equation}
\end{small}
where $\sigma(S)$ is the influence function, $D(S)$ represents the diversity of the influenced crowd; and $\gamma \in [0,1]$ is a trade-off parameter. $\hat{\mathcal{\sigma}}$ and  $\hat{{D}}$ are normalization factors, commonly
chosen as the maximum possible values of ${|\sigma(S)|}$ and ${D(S)}$,
respectively.

To demonstrate our diversified influence principle, we measure the diversity of two nodes $u,v\in \sigma(S)$ with their Euclidean distance over $k$-hop aggregated feature space: $D(u,v)=\frac{1}{2}\left\lVert
\frac{\mathbf{X}_u^{(k)}}{\lVert \mathbf{X}_u^{(k)} \rVert} - \frac{\mathbf{X}_v^{(k)}}{\lVert \mathbf{X}_v^{(k)} \rVert}\right\rVert $. Larger distance indicates
higher dissimilarity of two representations. We compute the diversity of nodes in $\sigma(S)$ by their average of pairwise-distance.
%\blue{
%\begin{equation}
%D_{pair}(S)=\frac{\sum_{u,v\in \sigma(S), u\neq v} \left\lVert
%\frac{\mathbf{X}_u^{(k)}}{\lVert \mathbf{X}_u^{(k)} \rVert} - \frac{\mathbf{X}_v^{(k)}}{\lVert \mathbf{X}_v^{(k)} \rVert} \right\rVert}{2|\sigma(S)|(|\sigma(S)|-1)}, %\mathbf{X}_v^{(k)}/{\lVert \mathbf{X}_v^{(k)}\rVert}\big,     
%\end{equation}
%}
%where $\mathbf{X}^{(k)}_v$ and $\mathbf{X}^{(k)}_u$ are the $k$-hop aggregated feature for nodes $v,u\in \mathcal{V}$, respectively.
Figure~\ref{fig:div} plots the relationship between influence diversity and the model accuracy when the influence magnitude $|\sigma(S)|$ is fixed to 400, demonstrating that the diversity significantly benefits the model performance.

\subsection{Submodular Diversity Functions}\label{sec:diveristy}
However, pair-wise diversity is not monotone and submodular. To guarantee diversity and efficient greedy
maximization, we next discuss a general recipe for constructing monotone and submodular diversity functions $D(S)$. Our scheme relies on enforcing the indirect influence by enabling each node in $\mathcal{V}$ can find a close node belonging to $\sigma(S)$. The diversity is enforced by avoiding repeatedly adding similar nodes to $\sigma{(S)}$ (i.e., those have already been affected by the indirect influence.)
Based on this scheme, we define two submodular diversity functions, i.e., the Nearest Neighbor(NN) based diversity and the coverage based diversity.

\begin{figure}[t]
\centering
\includegraphics[width=2.2in]{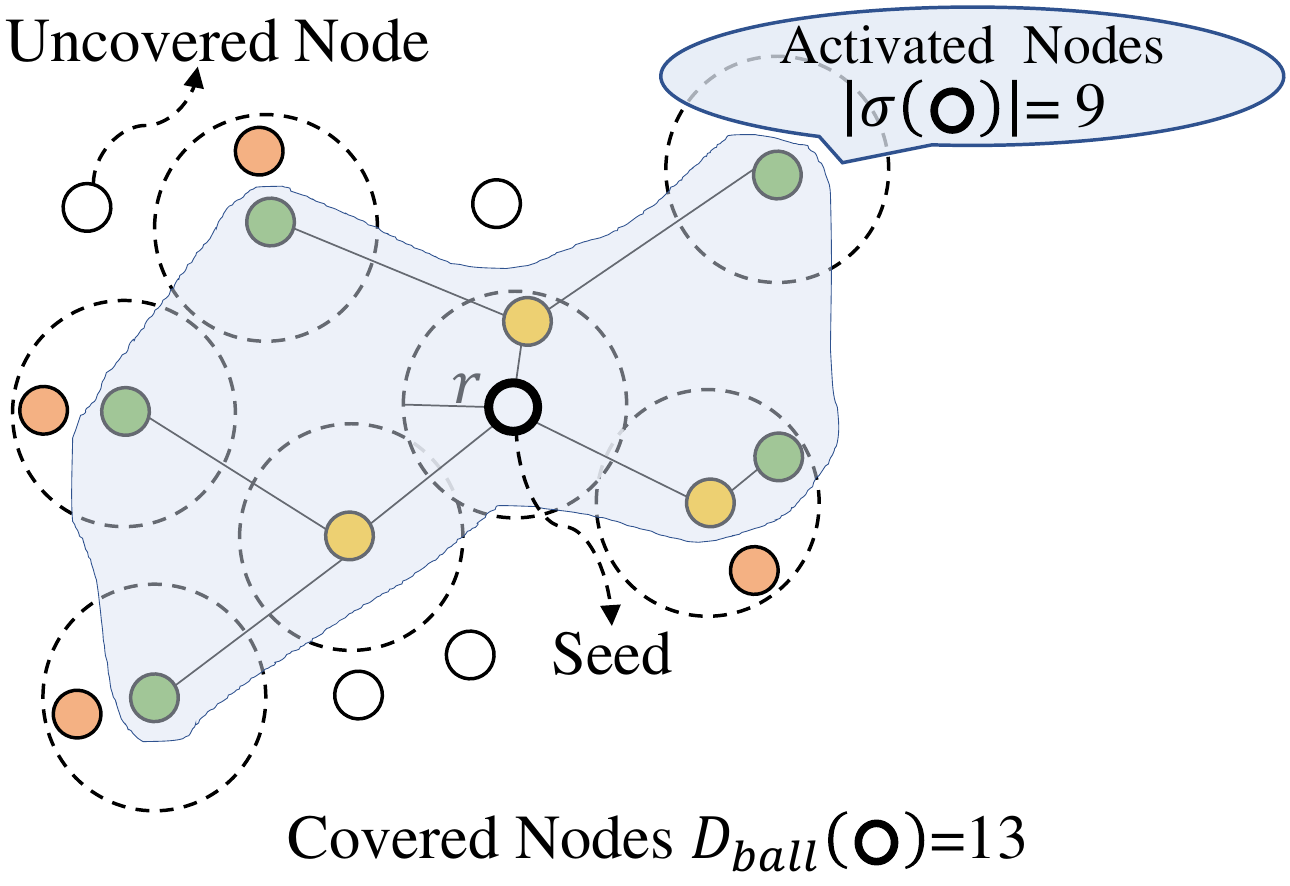}
% \vspace{-1em}
\caption{The coverage-based diversity of a seed node.}
\label{fig:cover}
\end{figure}

% \vspace{-1mm}
\textit{\underline{Nearest Neighbor (NN)-based Diversity.}} 
We measure the indirect influence for each node $v\in \mathcal{V}$ by $v$'s distance to the activated node $u\in \sigma(S)$ closest to it in the feature space.
Let $d(\cdot)$ measures the euclidean distance of two aggregated feature vectors, and $d_{max} =  \max_{u, v\in \mathcal{V}}\ d(\mathbf{X}^{(k)}_u, \mathbf{X}^{(k)}_v)$ is the maximum
pairwise distance. We define the diversity function as:
\begin{definition}[NN-diversity Function] Given the $k$-step feature propagation mechanism~\eqref{equ:prop}, the diversity function of seed set $S$ is:

\begin{footnotesize}
\begin{equation}
D_{NN}(S) = \sum_{u\in \mathcal{V}}\left({d_{max} - \min_{v\in \sigma(S)}\ d\left(\mathbf{X}^{(k)}_u, \mathbf{X}^{(k)}_v\right)}\right).
\label{eq:ds}
\end{equation}
\end{footnotesize}
\end{definition}

\begin{theorem}
\label{theorem}
The function $D_{NN}(S)$ is nondecreasing and submodular with respect to $S$.
\end{theorem}

% \vspace{-1mm}
\textit{\underline{Coverage-based Diversity.}}
Notice that $D_{NN}(S)$ is to minimize the total distance, it might incur relatively high performance variation without considering the variance of distance (and thus the indirect influence). So we further propose a coverage-based strategy emphasizing variance-reduction for more stable performance. Specifically, we assume indirect influence is only valid within each $r$-radius ball $G_u$ centered at each activated node $u\in \sigma(S)$: i.e., $ G_u = \left\{v\  \middle |\forall v \in \mathcal{V}, d(\mathbf{X}^{(k)}_u, \mathbf{X}^{(k)}_v)\leq r \right\}$. As illustrated in Figure~\ref{fig:cover}, our diversity function is the aggregated indirect influence of $\sigma(S)$:

\begin{definition}[Ball-diversity Function] 
Given a set of groups $\{G_u\}$ and $\sigma{(S)}$, the diversity function of seed set $S$ is: 
\begin{footnotesize}
\begin{equation}
D_{ball}(S) = \left|\bigcup_{u\in \sigma(S)} G_u \right|.
%\left|\left\{u\  \middle | \right\}\right|,
\label{C_t}
\end{equation}
\end{footnotesize}
\end{definition}

The influence function $|\sigma{(S)}$| can be treated as the special case of $D_{ball}(S)$ with ball radius $r = 0$, i.e., ignoring the indirect influence. 
Similar to $|\sigma{(S)}$|, $D_{ball}(S)$ is apparently increasing with larger $|S|$ since more nodes can be covered. Since more nodes have not been covered with larger $|S|$, it is harder to cover more nodes, and thus the marginal gain decreases accordingly. So we have:

\begin{theorem}
The function $D_{ball}(S)$ is nondecreasing and submodular with respect to $S$.
\end{theorem}
% \vspace{-1mm}
Note that the diversity function enforcing the \emph{indirect influence} of GNNs is completely new. The fundamental difference between our diversity functions and the classic approach is that the node coverage is determined by both the propagation over the graph and the distance over the feature space. For example, in our ball-diversity given by Eq.~\eqref{C_t}, the coverage of labeled node set $S$ in the feature space is the ball-covered regions centered on node set $\sigma(S)$ influenced by $S$ over the graph. 
By contrast, the classic coverage approach~\cite{DBLP:conf/nips/PrasadJB14} only considers the covered region centered on $S$, which is designed for independent and identically distributed (i.i.d) data and fails to model the influence imposed by the graph structure.

\begin{figure*}[tp]
\centering  
\subfigure[Cora]{
\label{Fig.single}
\includegraphics[width=0.28\textwidth]{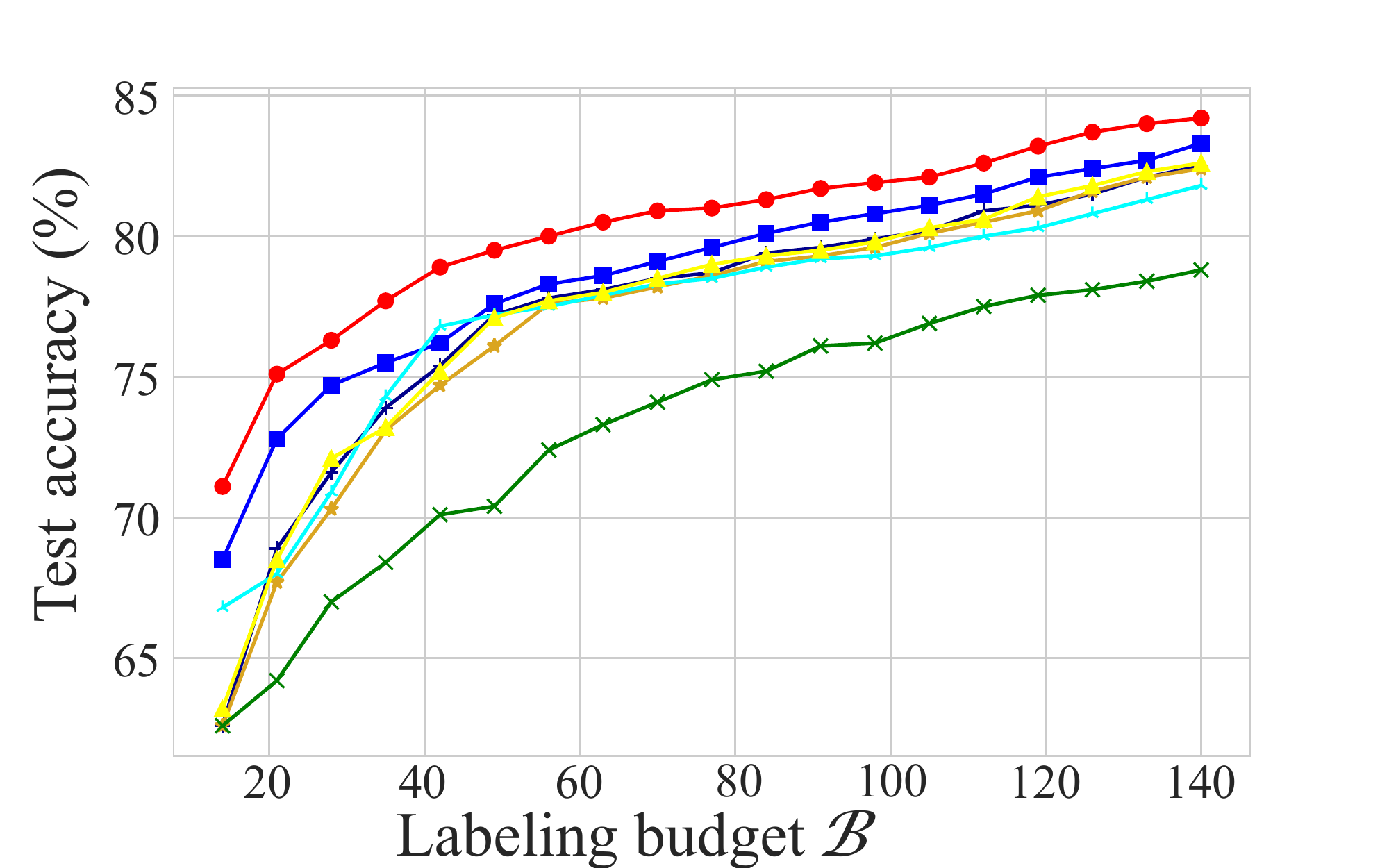}}
\subfigure[Citeseer]{
\label{Fig.ensemble}
\includegraphics[width=0.28\textwidth]{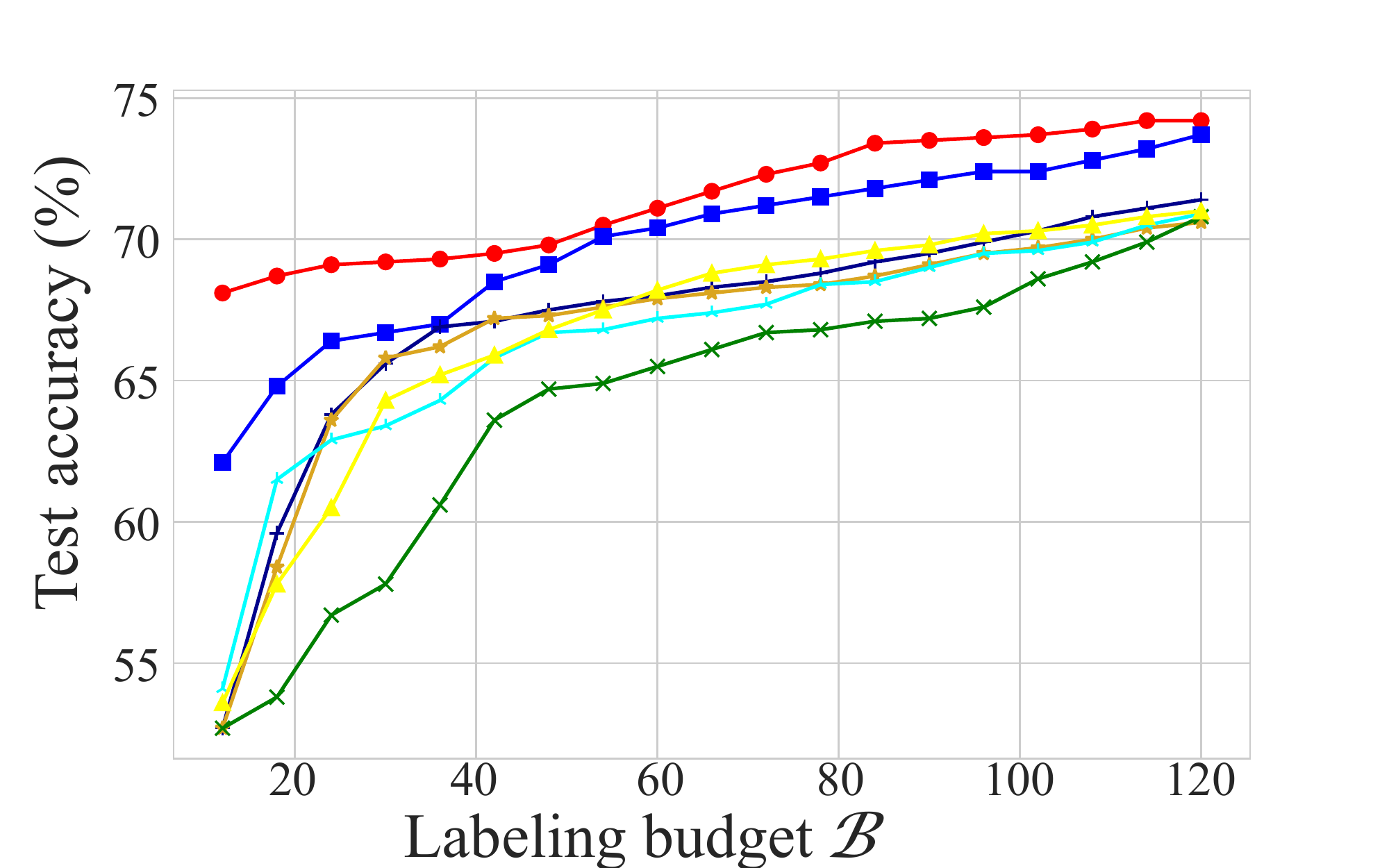}}
\subfigure[PubMed]{
\label{Fig.ensemble}
\includegraphics[width=0.28\textwidth]{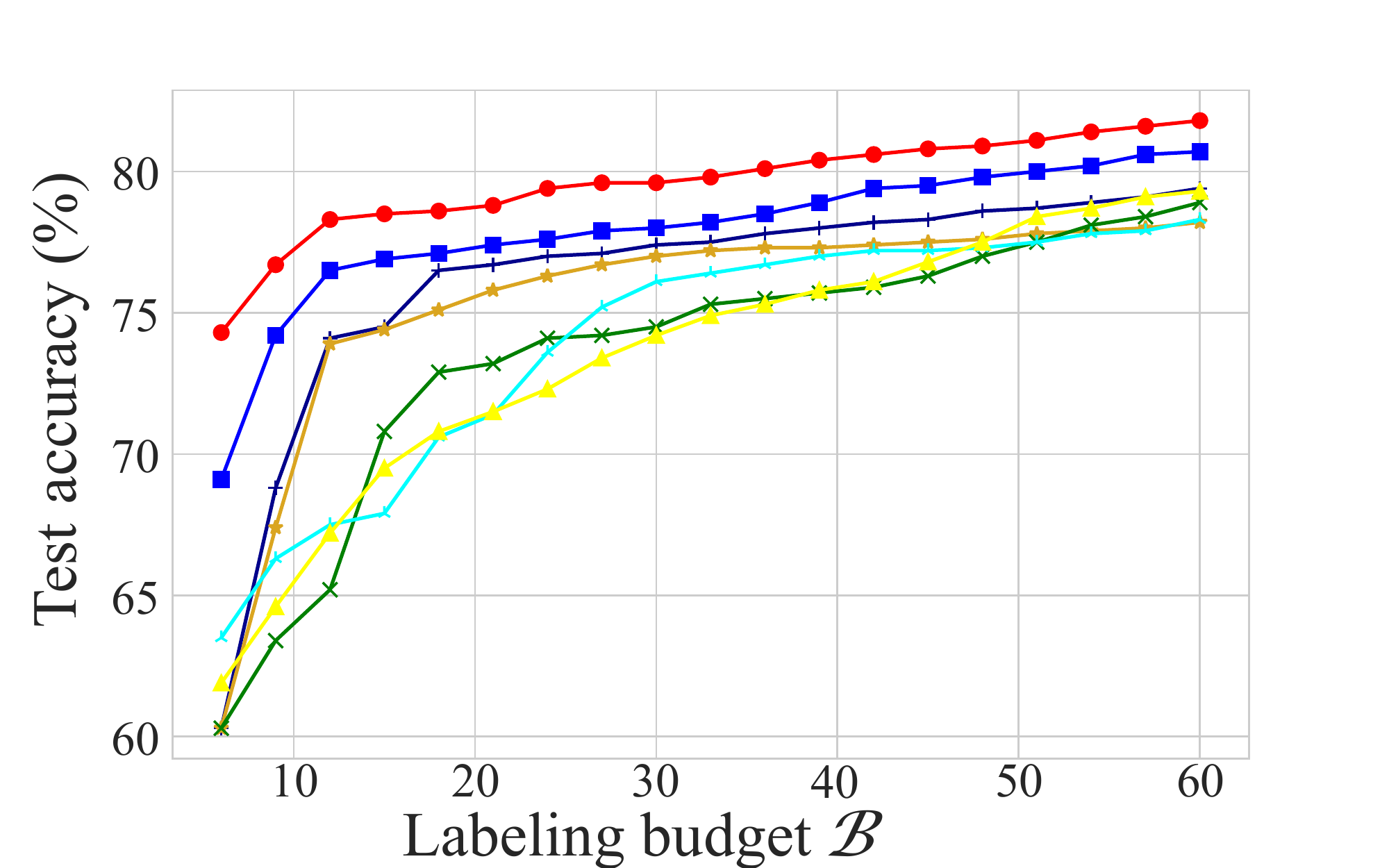}}
\subfigure{
\label{Fig.ensemble}
\includegraphics[width=0.08\textwidth]{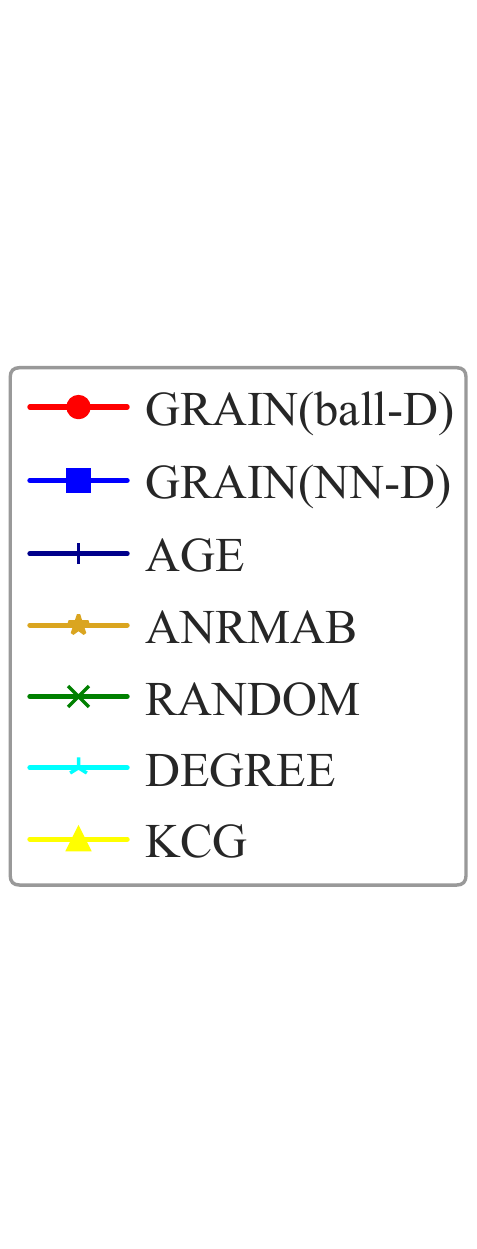}}
% \vspace{-5.5mm}
\caption{The test accuracy across different labeling budgets for model training.}
\label{fig.al_performance}
% \vspace{-4mm}
\end{figure*}

\begin{algorithm}[t]
  \caption{Greedy node selection}\label{alg:DSGD}
    \KwIn{Graph $\mathcal{G}$, feature $\mathbf{X}$, propagation mechanism $f$ and layer number $k$ for target GNN, budget $\mathcal{B}$.}
\KwOut{Seed set $S$}
    \For{$i=1,2,\ldots, k$}{
        $\mathbf{X}^{(k)} \gets f(\mathbf{X}^{(k-1)}, {\mathbf{T}}, \mathbf{X}^{(0)})$\;
    }
     $S=\emptyset$ \; 
    
    \For{$t=1,2,\ldots, \mathcal{B}$}{
        \For{$v\in \mathcal{V}_{train} \setminus S$}{
        update $\sigma(S \cup \{v\})$ and $D(S \cup \{v\})$ based on \eqref{mag} and \eqref{eq:ds} (or \eqref{C_t}), respectively\;
        }
        $v^*=\mathop{\arg\max}_{v\in \mathcal{V}_{train} \setminus S} \ \ {F(S \cup \{v\})} - F(S)$\;
             $S=S\cup\{{v}^*\}$\; 
        } 
     \textbf{return} $S$
\end{algorithm}

\subsection{Selection Algorithm}
% Both $|\sigma(S)|$ and $D(S)$ are nondecreasing and submodular, their non-negative linear combination of $F(S)$ in~\eqref{eq:obj} is also submodular and monotone with respect to $S$.
% As a result, it is possible for a simple greedy algorithm to achieve a near-optimal solution.
% \vspace{-1mm}

\textit{\underline{Greedy Algorithm.}} Algorithm 1 provides a sketch of our greedy node selection method for GNNs. 
Without losing generality, we consider a batch setting with $\mathcal{B}$ rounds where one node is selected in each iteration.
Given the propagation mechanism $f$ and layer $k$ inherited from a target GNN, we first perform the propagation based on equation~\eqref{equ:prop} (line 2).
Notice the marginal gain $F(S \cup \{v\})-F(S)$ of each node $v\in \mathcal{V}_{train} \setminus S$ is closely correlated to the current label set $S$, which is decreasing as the $S$ grows. Once a vertex is selected and added to the label set $S$, we update the marginal gain of each node $v\in \mathcal{V}_{train} \setminus S$ based on the new set $S$. Specifically, the influence and diversity score of $\sigma(S \cup \{v\})$ and $D(S \cup \{v\})$ would be updated according to Equations~\eqref{mag} and \eqref{eq:ds} (or \eqref{C_t}), respectively (lines 5-6).
Next, we select the node ${v}^*$ generating the maximum marginal gain, and the selected nodes set $S$ are updated (line 7-8). 
For monotone and submodular $F$, the final selected node set $S$ is within a factor of $(1-\frac{1}{e})$ of the optimal set $S^*$: $F(S) \geq (1-\frac{1}{e})F(S^{*})$.

\noindent\textit{\underline{Efficiency Optimization.}}
Compared to existing learning-based methods, \sys provides high efficiency and scalability advantages, as it avoids the training cost by separating the propagation mechanism from the neural networks. 
By leveraging the existing works on scalable and parallelizable social influence maximization, we could enable \sys to
effectively deal with large-scale graphs. The key idea is to identify and dismiss uninfluential nodes in order to dramatically reduce the amount of computation for evaluating influence spread. For example, we can use the degree of nodes or the distribution of random walkers throughout the nodes~\cite{kim2017scalable} to filter out a vast number of uninfluential nodes.

%Inspired by the properties of real-world graphs, \sys further allows making the following relaxation for extremely high efficiency. 
%First, real-world graphs exhibit power-law degree distributions, leading to highly skewed influence capability. 
%Thus, instead of greedily selecting nodes over the whole unlabeled set $\mathcal{V}_{train}$ (line 5 of the Algorithm~\ref{alg:DSGD}), we perform greedy selection over those with high PageRank values. 
%Specifically, we rank nodes in descending order of their PageRank values, select the top $m$ nodes as candidates, and then pick the one maximizing the objective function $F(S)$ from these candidates. 
%Second, most graphs exhibit the homophily phenomenon, so we can relax the
%problem from enforcing diversity on the influenced crowd to just enforcing diversity on the seed set. For example, we could relax $v\in \sigma(S)$ to $v\in S$ when computing diversity function given by equations~\eqref{eq:ds}  and~\eqref{C_t}. The idea behind this relaxation
%is that nodes with similar features are more likely to connect
%with each other. Therefore, if the seed set is diverse, the
%activated set would be diverse as well.

\section{Experiments}

\subsection{Experimental Settings}
\para{Datasets.} We evaluate \sys in both inductive and transductive settings~\cite{hamilton2017inductive} on three citation networks (i.e., Citeseer, Cora, and PubMed)~\cite{DBLP:conf/iclr/KipfW17}, one large social network (Reddit), and the largest benchmark dataset ogbn-papers100M~\cite{hu2021ogb}. 
% The properties of these datasets are summarized in Table~\ref{Dataset}.
More descriptions about the datasets are provided in Appendix A.2.

% \begin{itemize}
%     \item \textbf{Random}: Select nodes randomly;
%     \item \textbf{Degree}: Select nodes with maximum degree;
%     \item \textbf{AGE}~\cite{cai2017active}: Combine different query strategies linearly with time-sensitive parameters for GNNs;
%     \item \textbf{ANRMAB}~\cite{gao2018active}: Adopt a multi-armed bandit mechanism for adaptive decision making to select nodes for GNNs; 
%     \item \textbf{K-Center-Greedy} (KCG)~\cite{sener2018active}: Select nodes that maximize the distance from the nearest clustering center;
%     \item \textbf{\sys(ball-D)}: Use balls in semantic space to control the influence magnitude and diversity with the objective~\eqref{C_t};
%     \item \textbf{\sys(NN-D)}: Use the dynamic parameter $\lambda$ to combine influence magnitude and diversity with the objective~\eqref{eq:obj}.
% \end{itemize}
% Both AGE and ANRMAB are designed for active learning on GNNs, which adopt the uncertainty, density, and node degree to select nodes. 
% They have to train the evaluated model since both the uncertainty and density are calculated based on the model predictions. 
% By contrast, \sys considers this problem from the perspective of diversified influence maximization and the selection process is model-free.

\para{GNN Models.}We conduct experiments using the widely used GCN model and
demonstrate the generalization of \sys on other GNNs such as SGC~\cite{wu2019simplifying}, APPNP~\cite{klicpera2018predict} and MVGRL~\cite{hassani2020contrastive} in Section~\ref{gen}.
The description of these GNNs is provided in Appendix A.3. 

\para{Baselines.} We compare \sys(ball-D) and \sys(NN-D) with the following baselines: Random, Degree, AGE~\cite{cai2017active}, ANRMAB~\cite{gao2018active}, K-Center-Greedy (KCG)~\cite{sener2018active}. 
A detailed introduction of these baseline methods can be found in Appendix A.5.

% \begin{table}[tbp]
% \caption{\small Properties of the four datasets. (T and I indicate transductive and inductive respectively).} 
% \centering
% {
% \noindent
% \renewcommand{\multirowsetup}{\centering}
% \resizebox{0.95\linewidth}{!}{
% \begin{tabular}{cccccccc}
% \toprule
% \textbf{Dataset}&\textbf{\#Nodes}& \textbf{\#Features}&\textbf{\#Edges}&\textbf{\#Classes}&\textbf{\#Train/Val/Test}&\textbf{\#Type}\\
% \midrule
% Cora& 2708 & 1433 &5429&7&1208/500/1,000&T \\
% Citeseer& 3327 & 3703&4732&6& 1827/500/1,000&T \\
% PubMed& 19717  & 500 &44338&3& 18217/500/1,000&T \\
% \blue{ogbn-papers100M} & \blue{111,059,956} & \blue{128} & \blue{1,615,685,872} & \blue{172} &
% \blue{1,207K/125K/214K}&\blue{T}\\
% Reddit& 232,965 & 602 & 11,606,919 & 41 & 
% 155K/23K/54K
% & I 
% \\
% \bottomrule
% \end{tabular}}}
% \label{Dataset}
% % \vspace{-2mm}
% \end{table}

\para{Settings.}
For each method, we use the hyper-parameter tuning toolkit~\cite{li2021openbox, li2021volcanoml} or follow the original papers to find the optimal hyperparameters. To eliminate randomness, we repeat each method ten times and report the mean performance. Specifically, each learning-based active learning strategy
chooses a small set of labeled nodes as an initial pool. Like AGE,
we consider the label balance, and two nodes are randomly selected
for each class. Note that \sys can guide the node selection from scratch under the new diversified influence maximization criterion.
The parameter settings, implementation detail, and reproduction instructions can be found in Appendix A.4.

\subsection{Performance Comparison}

In this section, we compare the performance of \sys with the aforementioned baselines in two typical data selection scenarios: active learning and core-set selection.

\para{Active Learning.}
Let $C$ be the number of classes for each dataset (e.g., 7 for Cora and 3 for PubMed). We choose the budget $\mathcal{B}=|S|$ from a range of $2C$ to $20C$ labeled nodes, and report the test accuracy of the GCN model trained on the selected labeled node set $S$ along with the number of labeled nodes for training in Figure~\ref{fig.al_performance}. As the labeling cost is proportional to the labeling budget, Figure~\ref{fig.al_performance} equivalently shows the improvement in terms of labeling cost.

Compared to the other baselines, both \sys(ball-D) and \sys(NN-D) quickly boost the accuracy at the beginning and consistently outperform the baselines as the number of labeled nodes grows.
Concretely, the competitive baseline AGE has to label 120 nodes to achieve the accuracy of 71.4\% on Citeseer, while \sys(NN-D) only needs 60 labeled nodes to achieve similar results, indicating that \sys could cut the cost (e.g., money) by half for users.
This improvement demonstrates the effectiveness of our DIM selection criterion. 
Moreover, \sys can avoid the sensitivity to model accuracy inherited from the learning-based methods like AGE and ANRMAB, especially when the model is under-fitted given a small labeling budget. Concretely, \sys(ball-D) outperforms AGE by a margin of 8.8\% on Citeseer when 18 labeled nodes are used.

To demonstrate the improvement of \sys on the final performance, we also provide the test accuracy using all the $20C$ labeled nodes.
Table~\ref{accuracy_120} shows that AGE and ANRMAB outperform the Random and Degree method on most graph datasets, and \sys(ball-D) and \sys(NN-D) further boost the performance by a significant margin.
\sys(ball-D) improves the test accuracy of the best baseline AGE by 1.7-2.8\% on the three citation networks while \sys(NN-D) also outperforms AGE by a margin of 0.9\% on Reddit. 
Note that \sys(ball-D) outperforms \sys(NN-D) on three citation networks, while \sys(NN-D) performs better than \sys(NN-D) on Reddit. 
It is because the citation networks have low degrees (i.e., are sparse). Thus the model performance is more vulnerable to the variance of node feature distance (indirect influence). 
\sys(NN-D) performs better on dense graphs (Reddit) since it can minimize the total distance, while \sys(ball-D) outperforms \sys(NN-D) on sparse graphs by emphasizing variance-reduction.

We also conduct an experiment on the largest benchmark dataset ogbn-papers100M. For such large-scale graphs, the learning-based methods such as AGE and ANRMAB require extremely long training time. In our experiment, both AGE and ANRMAB fail to finish the training within two weeks, and this indicates that \sys(26.1 hours) could achieve at least an order of magnitude speedup. Besides, we observe that \sys(ball-D) outperforms the second-best method KCG by a large margin of 1.3\%.

% \begin{table}[tpb]
% \caption{Test accuracy using $20C$ labeled nodes in active learning scenario. \blue{OOT means ``out of time''.}}
% % \vspace{-3mm}
% \centering
% {
% \noindent
% \renewcommand{\multirowsetup}{\centering}
% \resizebox{0.95\linewidth}{!}{
% \begin{tabular}{cccccccc}
% \toprule
% \textbf{Method}& 
%  \textbf{Cora}& \textbf{Citeseer}& \textbf{PubMed}& \textbf{Reddit}& \textbf{\blue{ogbn-papers100M}}\\
% \midrule
% Random&78.8&70.8&78.9&91.1&\blue{51.2} \\
% Degree&81.8&70.9&78.3&91.4&\blue{51.5} \\
% AGE&82.5&71.4&79.4&91.6&\blue{OOT}\\
% ANRMAB&82.4&70.6&78.2&91.5&\blue{OOT} \\
% KCG&82.6&71&79.3&91.3&\blue{51.6}\\
% % GEEM&1e5&-&-&- \\
% \hline
% \sys(NN-D)&83.3&73.7&80.8&\textbf{92.5}&\blue{52.6} \\
% \sys(ball-D)&\textbf{84.2}&\textbf{74.2}&\textbf{81.8}&92.3&\blue{\textbf{52.9}}\\
% \bottomrule
% \end{tabular}}}
%  \label{accuracy_120}
% \end{table}
\begin{figure}[tp]
% \vspace{-4.5mm}
\centering  
\subfigure{
\label{Fig.ensemble}
\includegraphics[width=0.25\textwidth]{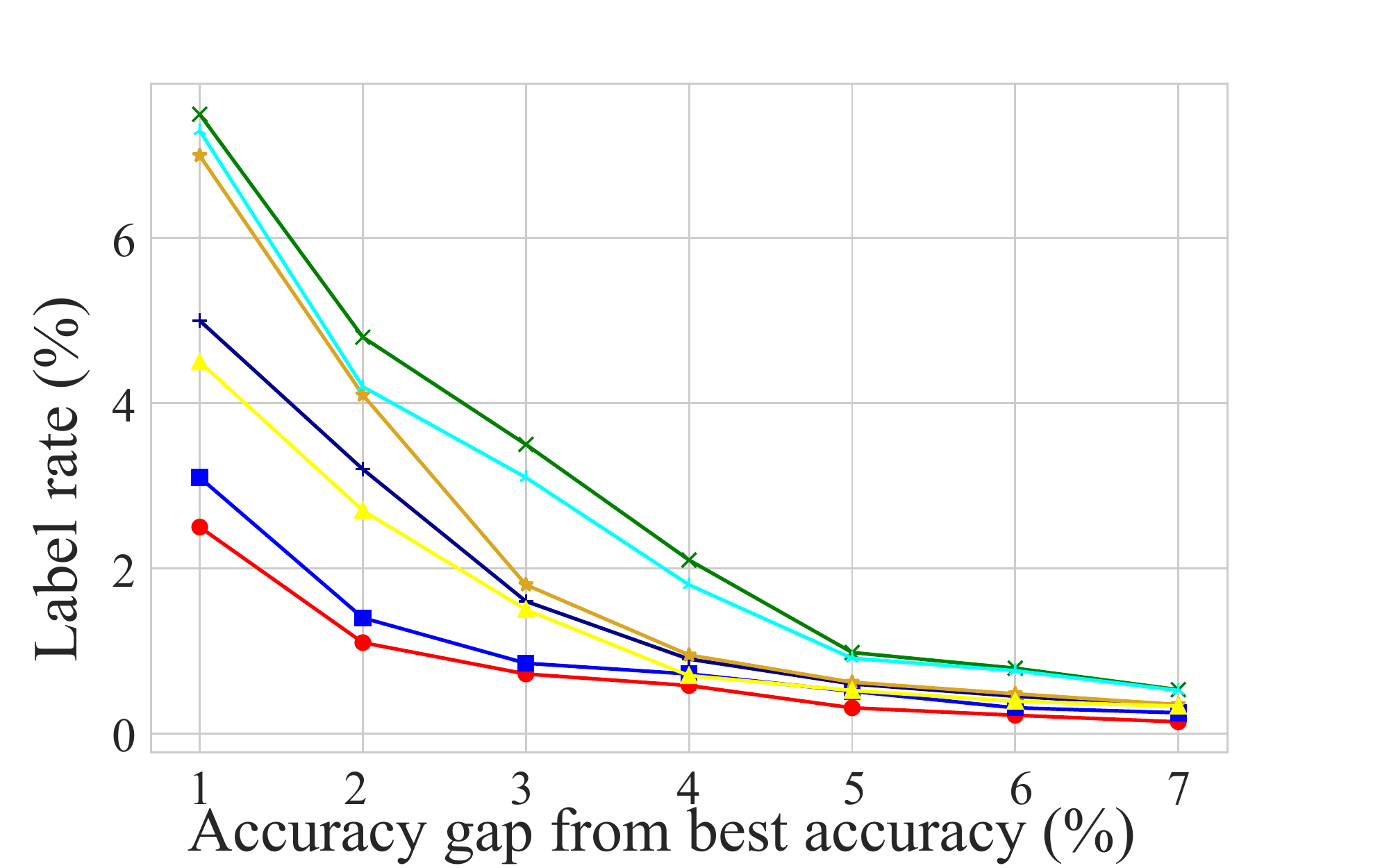}}
\subfigure{
\label{Fig.ensemble}
\includegraphics[width=0.08\textwidth]{figure/legend1.pdf}}
\vspace{-3.5mm}
\caption{The label rate (the percentage of the whole available node set) to reach a accuracy gap from the best accuracy.}
\label{fig.coreset_performance}
% \vspace{-4mm}
\end{figure}

\begin{table}[tpb]
\caption{Test accuracy using $20C$ labeled nodes in active learning scenario. OOT means ``out of time''.}
% \vspace{-3mm}
\centering
{
\noindent
\renewcommand{\multirowsetup}{\centering}
\resizebox{0.95\linewidth}{!}{
\begin{tabular}{cccccccc}
\toprule
\textbf{Method}& 
 \textbf{Cora}& \textbf{Citeseer}& \textbf{PubMed}& \textbf{Reddit}& \textbf{ogbn-papers100M}\\
\midrule
Random&78.8&70.8&78.9&91.1&51.2 \\
Degree&81.8&70.9&78.3&91.4&51.5 \\
AGE&82.5&71.4&79.4&91.6&OOT\\
ANRMAB&82.4&70.6&78.2&91.5&OOT \\
KCG&82.6&71&79.3&91.3&51.6\\
% GEEM&1e5&-&-&- \\
\hline
\sys(NN-D)&83.3&73.7&80.8&\textbf{92.5}&52.6 \\
\sys(ball-D)&\textbf{84.2}&\textbf{74.2}&\textbf{81.8}&92.3&\textbf{52.9}\\
\bottomrule
\end{tabular}}}
 \label{accuracy_120}
\end{table}

\para{Core-set Selection.}
Core-set selection starts with a large labeled or unlabeled dataset and aims to find a small subset that accurately approximates the entire dataset.
We first train the exact model with all training labels (i.e., 18217) and get the test accuracy (86.5\%) on PubMed.
As shown in Figure~\ref{fig.coreset_performance}, we evaluate the number of labeled nodes for each core-set selection method to achieve the corresponding accuracy gap. 
It is evident that both \sys(ball-D) and \sys(NN-D) significantly outperform the other baselines and achieve the same accuracy gap using much fewer labeled nodes.
Concretely, to achieve an accuracy gap of 2\% on PubMed, AGE needs 3.2\% of all labeled training nodes while \sys(ball-D) needs only 1\% of these nodes, which means \sys(ball-D) outperform AGE by 3.2$\times$ in terms of data efficiency. 
More results on the other graphs are provided in Appendix A.7.

\begin{figure}[tp]
% \vspace{-4mm}
\centering   
\subfigure[Small-scale datasets]{
\label{fig.effi_gpu}
\includegraphics[width=0.252\textwidth]{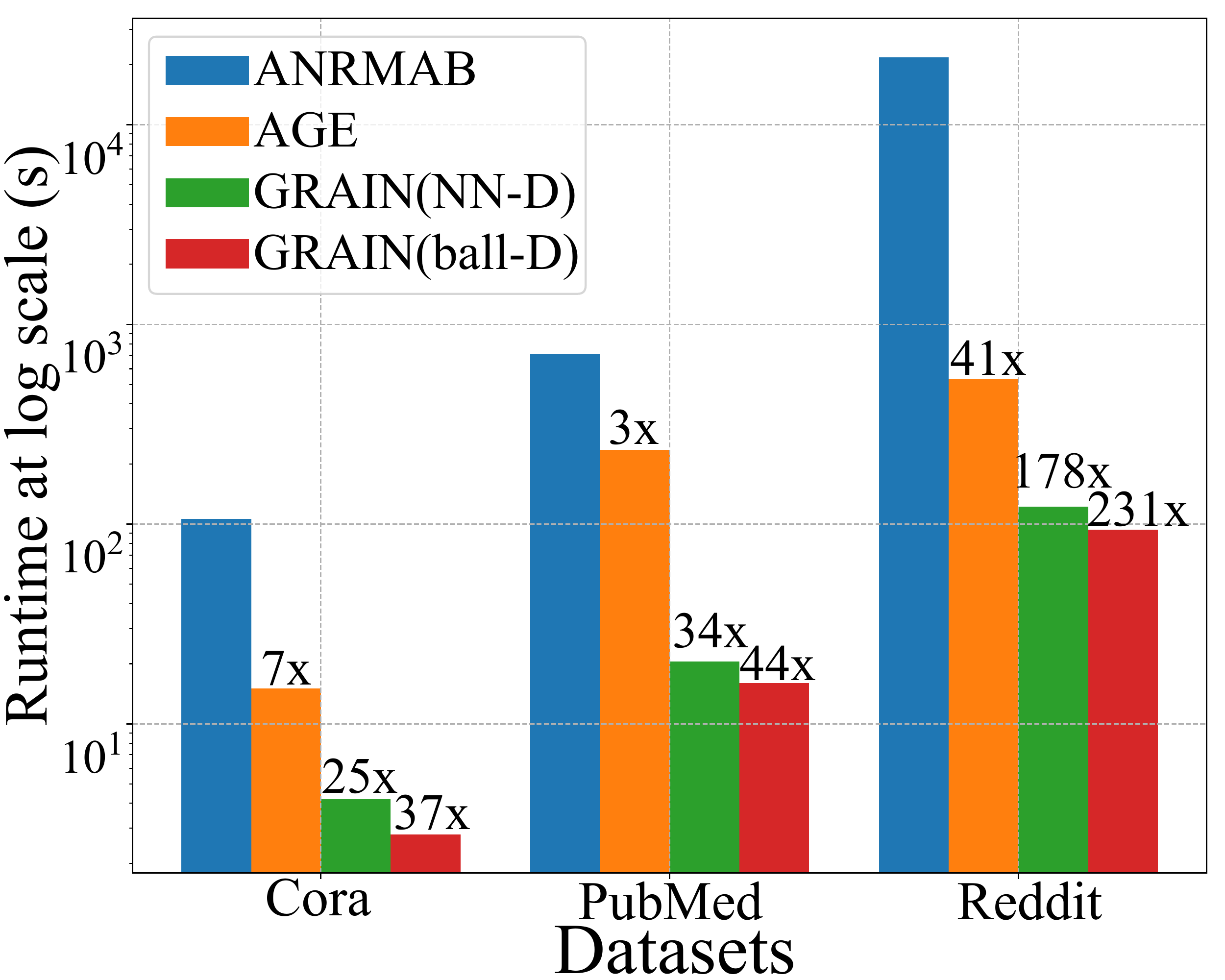}}
\subfigure[{Large ogbn-papers100M}]{
\label{fig.effi_nogpu}
\includegraphics[width=0.208\textwidth]{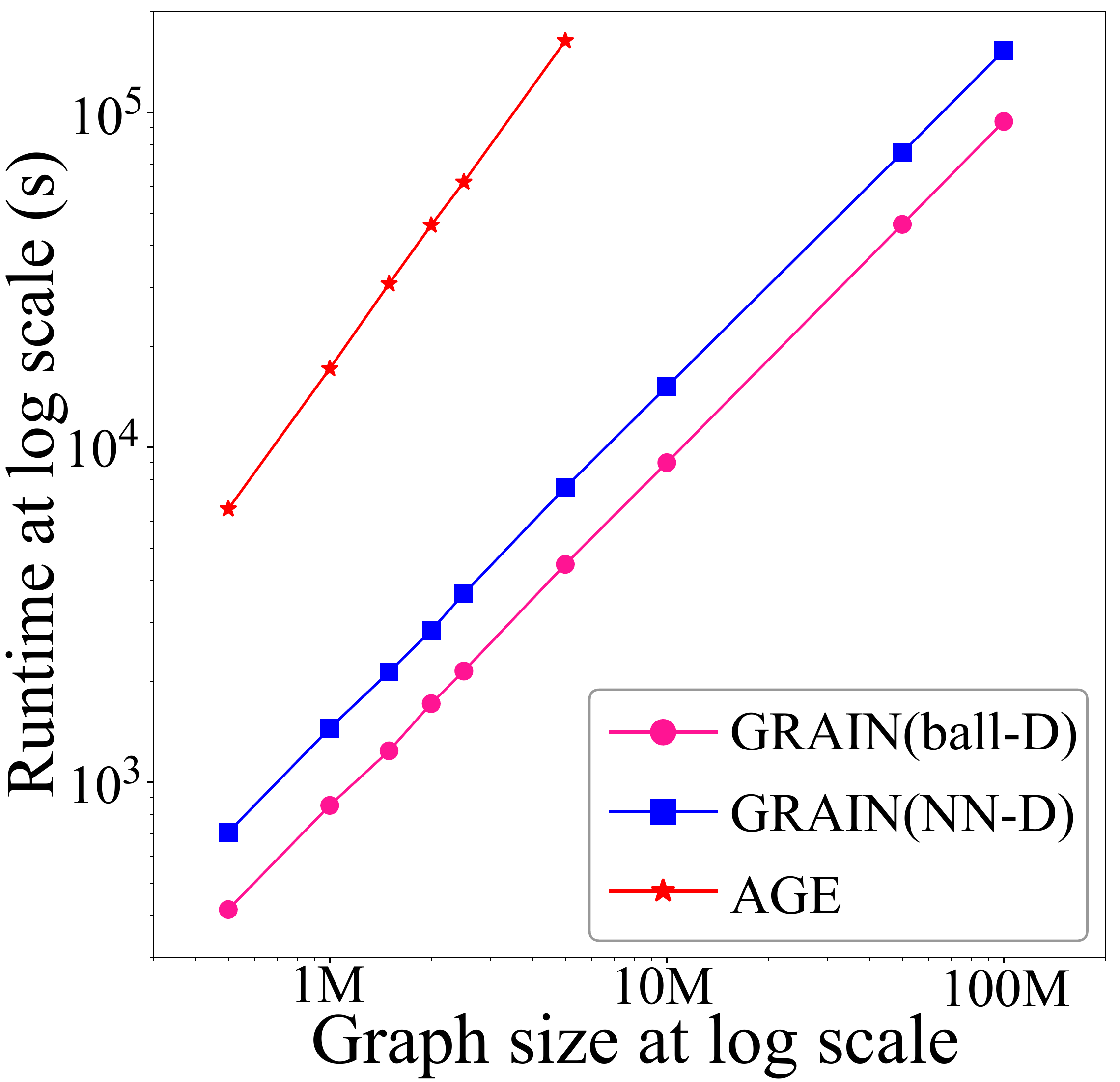}}
\vspace{-4.5mm}
\caption{End-to-end runtime (at log scale) using GPU. }
\label{speed}
% \vspace{-4mm}
\end{figure}

\subsection{Efficiency and Scalability Analysis}
Another advantage of \sys is its high efficiency in node selection. We report the end-to-end runtime of active learning methods in Figure~\ref{speed}. 
Note that the end-to-end runtime results include the overhead from both the node importance measurement and node selection, and do not consider the human-intensive Oracle labeling time which depends on the complexity of a specific task and monetary budget. 
Moreover, \sys is oracle-free and model-free, that is, the labeling process is not required in the node selection. For comparison, the learning-based process of AGE and ANRMAB has to wait for the oracle to provide labels in each iteration.

%In real-world AL settings, the end-to-end runtime of model-based methods (i.e., AGE and ANRMAB) includes the model training time, nodes selection time, and labeling time.
%However, the human-intensive labeling time highly depends on the complexity of a specific task and monetary budget. For example, the more expensive expert will lead to less labeling time for the same nodes. To simplify the comparison, the labeling time is excluded from our measurement.
% In real-world AL settings, the end-to-end runtime of model-based methods (i.e., AGE and ANRMAB) includes the labeling time and model training time, but the labeling time is excluded from our measurement since it is highly influenced by the quality and proficiency of oracles.
% As a special neural network, GCN contains lots of matrix computations and the training time can be significantly decreased by using GPUs.
Figure~\ref{fig.effi_gpu} shows that \sys(ball-D) obtains a speedup of 37×, 44×, and 231× over ANRMAB on Cora, PubMed, and Reddit respectively on GPU.
To test the scalability on large-scale graphs, we sample different scales on ogbn-papers100M. 
Since training GCN on such a large graph will lead to the out-of-memory exception, instead, we use SGC as the training model for both AGE and \sys.
Figure~\ref{fig.effi_nogpu} demonstrates \sys achieves a linear scaling on the ogbn-papers dataset, and at least one order of magnitudes faster than AGE. 
To select the same number of labeled nodes on the 100M graph scale, 
\sys(ball-D) and \sys(NN-D) take 26.1 hours and 42.5 hours respectively, while it takes AGE more than one year to achieve this by estimating its runtime trend. 

\begin{table}[tpb]
% \vspace{-4mm}
\caption{\small Influence of different components.}
% \vspace{-2mm}
\centering
{
\noindent
\renewcommand{\multirowsetup}{\centering}
\resizebox{0.95\linewidth}{!}{
\begin{tabular}{ccccccc}
\toprule
    \textbf{Method}&\textbf{Cora}&$\Delta$&\textbf{Citeseer}&$\Delta$&\textbf{PubMed}&$\Delta$\\
\midrule
No Magnitude& 81.1 &-3.1 & 70.8 &-3.4 & 76.7 &-5.1\\
No Diversity& 82.2 &-2.0 & 71.2 &-3.0 & 79.9 &-1.9\\
Classic Coverage &82.3 &-1.9 &71.5 &-2.7 &80.2 & -1.6 \\
\midrule
\textbf{\sys(ball-D)}& \textbf{84.2} &--& \textbf{74.2} &--& \textbf{81.8} &--\\
\bottomrule
\end{tabular}}}
\label{Ablation}
% \vspace{-0.5mm}
\end{table}

\begin{table}[tpb]
\vspace{-3mm}
\caption{Test accuracy of different models on PubMed.}
\vspace{-2mm}
\centering
{
\noindent
\renewcommand{\multirowsetup}{\centering}
\resizebox{0.6\linewidth}{!}{
\begin{tabular}{ccccccc}
\toprule
\textbf{Method}& 
 \textbf{SGC}& \textbf{APPNP}& \textbf{GCN}& \textbf{MVGRL}\\
\midrule
Random&77.6&79.2&78.9&79.3 \\
Degree&77.3&78.6&78.3&78.7\\
AGE&78.8&79.9&79.4&79.9 \\
ANRMAB&77.8&78.7&78.2&78.9 \\
KCG&78.2&79.7&79.3&79.8\\
\hline
\sys(NN-D)&80.2&81.6&80.8&81.8 \\
\sys(ball-D)&\textbf{81.1}&\textbf{82.0}&\textbf{81.8}&\textbf{82.1}\\
\bottomrule
\end{tabular}}}
 \label{accuracy_model}
\end{table}

\subsection{Ablation Study}
Our method combines both the influence magnitude and the diversity measures. To verify the necessity of each component, we evaluate \sys(ball-D) while disabling one measure at a time.
We evaluate \sys(ball-D): {\it (i)} without the diversity and only maximize $|\sigma(S)|$ (called "No Diversity"); {\it (ii)} without the influence magnitude and the goal is to cover the maximum number of nodes with the balls generated from selected nodes $S$ (called "No  Magnitude"); {\it (iii)} replace $|\sigma(S)|$  with $S$ when computing diversity (called "Classic Coverage").
Table~\ref{Ablation} displays the results of these methods.

\para{Influence Magnitude.} The influence magnitude has a significant impact on model performance on all datasets, and it is more important than diversity since removing it will lead to a significant performance gap. For example, the gap on PubMed is 5.1\%, which is much higher than the other gap (1.9\%). The higher the influence magnitude is, the more labeled nodes we can use to train a GNN.

\para{Influence Diversity.} 
The test accuracy decreases in all three datasets if we ignore the influence diversity.
For example, the performance gap is as large as 3.0\% if the influence diversity is removed on Citeseer.
% Although the gap is relatively smaller than that of the influence magnitude, it is still significant enough and can not be ignored.
The higher the influence diversity is, the more nodes are influenced in the aggregated feature space.

To further demonstrate the novelty of our diversity function, we add a baseline that adopts a classic coverage approach for diversity measurement. The result shows the test accuracy of \sys has decreased by a large margin if we replace $\sigma(S)$ with $S$ in the ball-diversity function, which verifies the necessity of considering propagation in the diversity measurement.

\begin{figure}[tp]
% \vspace{-4mm}
\centering  
\subfigure[\sys(ball-D)]{
\label{fig:inter_grain}
\includegraphics[width=0.23\textwidth]{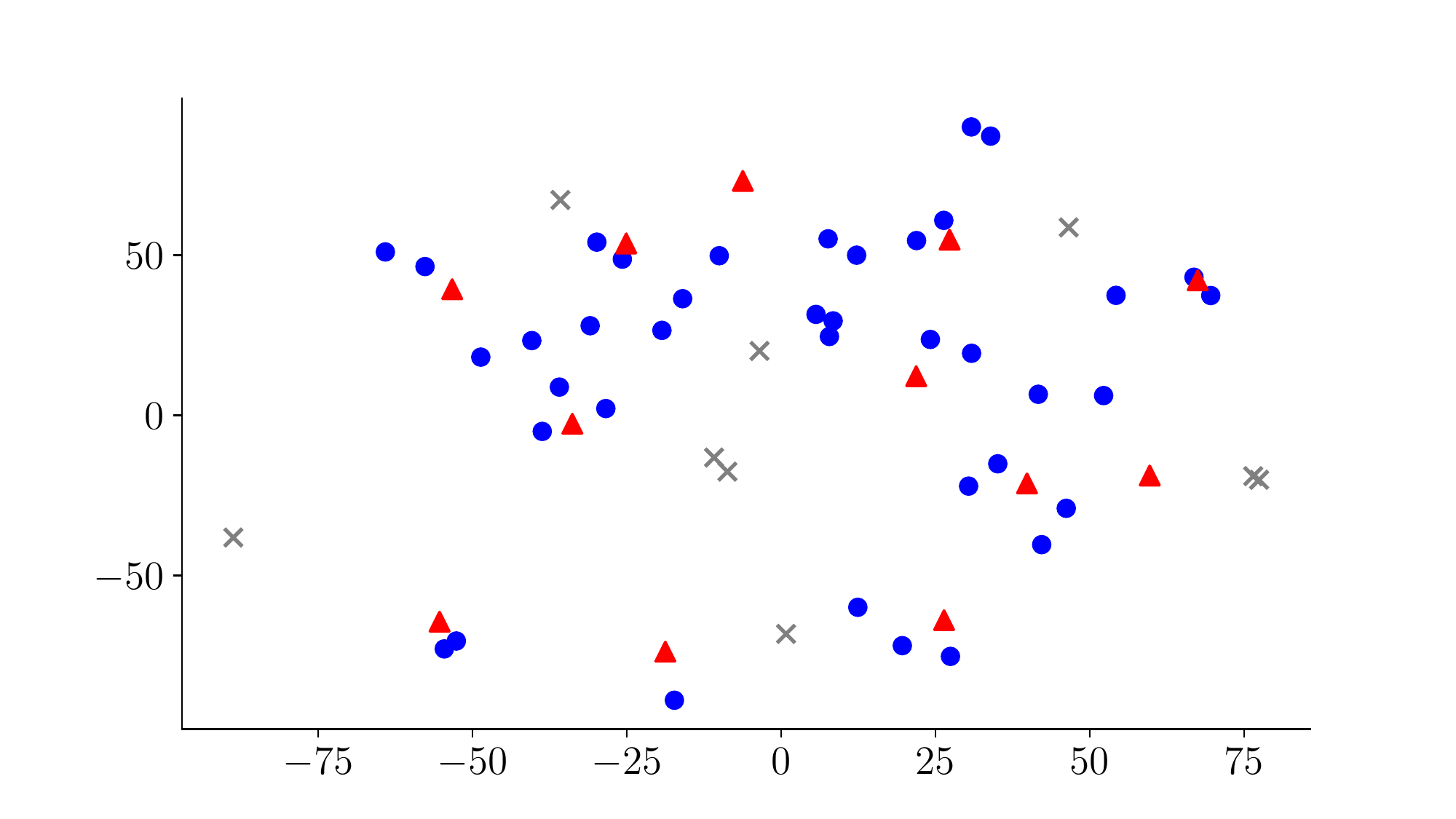}}
\subfigure[AGE]{
\label{fig:inter_age}
\includegraphics[width=0.23\textwidth]{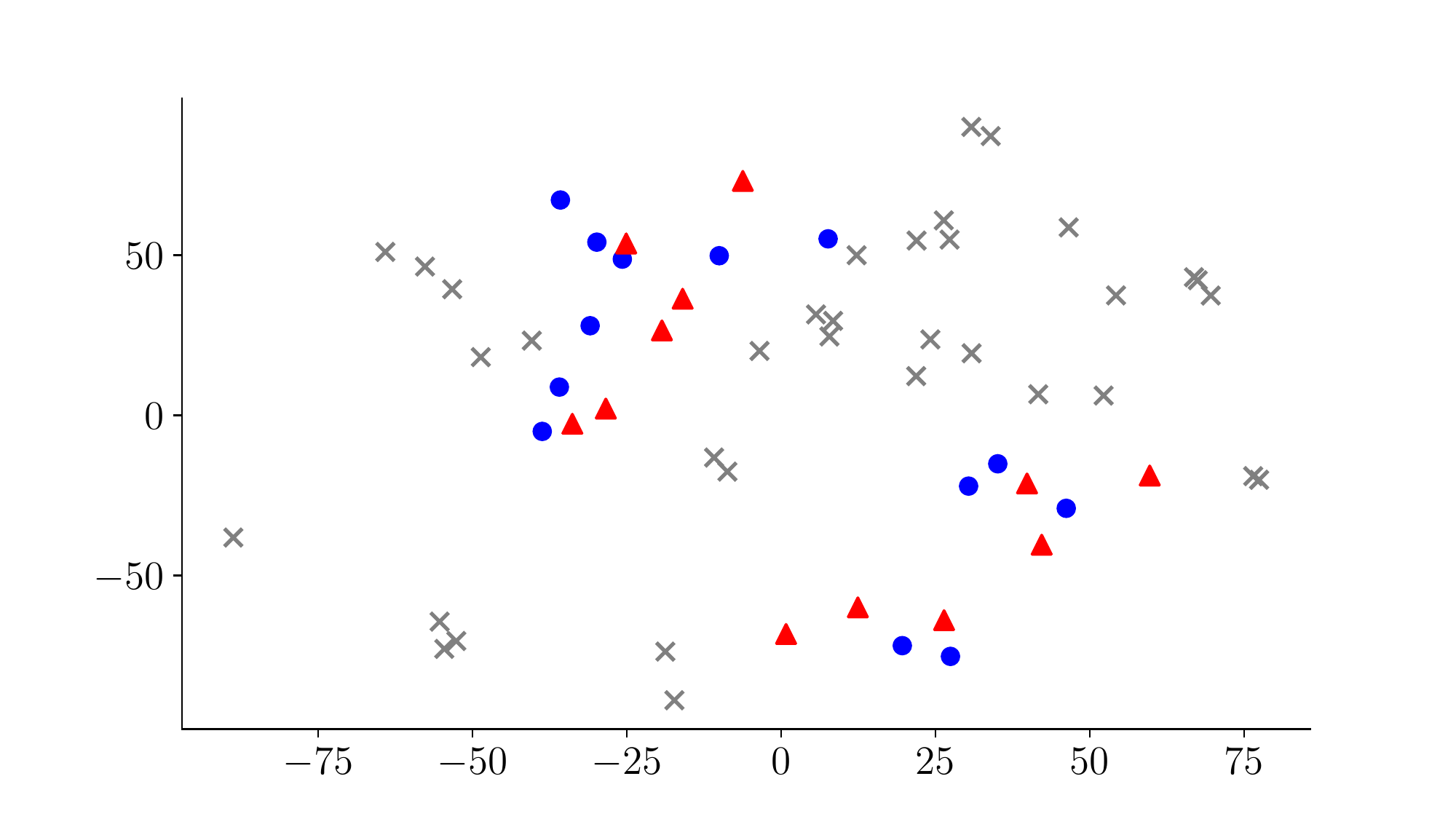}}
% \vspace{-3mm}
\caption{Distribution of the seed nodes and activated nodes selected by \sys(ball-D) and AGE. The red triangle refers to the seed nodes, the blue circle refers to the activated nodes, and the gray x refers to the non-activated nodes.}
\label{interpretability}
%\vspace{-2mm}
\end{figure}

\subsection{Generalization}
\label{gen}
In addition to GCN, \sys can also be applied to a large variety of GNN variants.
GCN, SGC, MVGRL and APPNP are four representative GNNs~\cite{chen2020graph} which adopt different message passings. Unlike the coupled GCN, both SGC and APPNP are decoupled, while their orderings when doing feature propagation and transformation are different. Besides, MVGRL is a classic self-supervised GNN.
We test the generalization ability of \sys by evaluating the aforementioned four types of GNNs on $20C$ nodes selected by \sys and other baselines in the AL scenario, and the corresponding results are shown in Table~\ref{accuracy_model}. 
The results suggest that both \sys(ball-D) and \sys(NN-D) consistently outperform the other baselines, regardless of the coupled and decoupled GNNs.
Moreover, the result also shows our proposed method \sys can significantly outperform the compared baselines on the top of more sophisticated self-supervised GNN such as MVGRL~\cite{hassani2020contrastive}. As shown in the table, the test accuracy of \sys(ball-D) could outperform KCG and AGE by more than 2\% on PubMed.
Therefore, we conclude that \sys can generalize to different types of GNNs well.

% \vspace{0.5mm}
\subsection{Model Interpretability}
In this section, we illustrate the insight of our method and conduct experiments to explain why \sys obtains better performance from the perspective of influence magnitude and influence diversity.
To depict the data distribution, we randomly choose 60 nodes from Citeseer and select 12 of them using \sys(ball-D) and AGE as the labeled set.
Then we mark the seed nodes, activated nodes, and non-activated nodes in Figure~\ref{interpretability} and use t-SNE~\cite{maaten2008tsne} to visualize them in the aggregated feature space.

For influence magnitude |$\sigma(S)$|, we observe that the number of non-activated nodes of AGE is larger than \sys, which means more unlabeled nodes are activated and get involved in the model training using \sys(ball-D). 
For influence diversity $D(S)$, it is evident that the activated nodes of \sys(ball-D) scatter over different regions of the whole dataset, while the nodes gather in some specific area of AGE. In this way, more unlabeled nodes in \sys(ball-D) can be affected by the indirect influence from $\sigma(S)$, and the performance is boosted accordingly.

\section{Conclusion}
GNNs are emerging deep learning models that arise naturally from the requirements of applying neural network models on graphs.
Efficient and scalable data selection for GNN training is demanding but still challenging due to its inherent complexity. 
This paper advocates a novel perspective for GNN data selection by connecting it with social influence maximization.
\sys represents a critical step in this direction by showing both the feasibility and potential of this connection.
To this end, we define a new feature influence model to exploit the common patterns of GNNs and propose novel submodular influence and diversity functions.
Experiments show that \sys outperforms competitive baselines by a large margin in terms of both model performance and efficiency.

% \begin{acks}
% This work is supported by the National Key Research and Development Program of China (No. 2018YFB1004403), NSFC (No. 61832001, 6197200), Beijing Academy of Artificial Intelligence (BAAI), and PKU-Tencent joint research Lab. Wentao Zhang and Zhi Yang contributed equally, and Bin Cui is the corresponding author.
% \end{acks}

% \clearpage
\appendix
\section{Outline}
This supplemental material is organized as follows:
\begin{description}
    \item[A.1] Proof.
    \item[A.2] Datasets description.
    \item[A.3] More details about the compared GNNs.
    \item[A.4] Concrete settings.
    \item[A.5] More details about the other node selection methods.
    \item[A.6] Efficiency comparison on CPU.
    \item[A.7] Core-set selection on other graphs. 
\end{description}

\para{A.1\ \  Proof}

\begin{theorem}
\label{theorem}
Function $D_{NN}(S)$ is monotone and submodular with respect to $S$.
\begin{equation}
    D_{NN}(S) = \sum_{u\in \mathcal{V}}\left({d_{max} - \min_{v\in \sigma(S)}\ d\left(\mathbf{X}^{(k)}_u, \mathbf{X}^{(k)}_v\right)}\right)
\end{equation}

\end{theorem}
\begin{proof}
By definition, $\sigma(S) \subseteq \sigma(T)$ if $S \subseteq T$. We first prove the following lemmas:
\begin{lemma}
\label{lemma:nondecreasing}
Function $f(S)=\min_{v \in \sigma(S)}d\left(\mathbf{X}^{(k)}_u, \mathbf{X}^{(k)}_v\right)$ is nonincreasing given node $u$ with respect to $S$.
\end{lemma}
\begin{proof}
Thus, $\forall S\subseteq T$,
\begin{equation}
\begin{aligned}
    &f(T)-f(S)\\
    =&\min\left(f(S),\min_{v \in \sigma(T)\backslash \sigma(S)}d\left(\mathbf{X}^{(k)}_u, \mathbf{X}^{(k)}_x\right)\right)-f(S)\\
    \leq&\ 0.
\end{aligned}
\end{equation}
\end{proof}

\begin{lemma}
\label{lemma:submodular}
Function $f(S)=-\min_{v \in \sigma(S)}d\left(\mathbf{X}^{(k)}_u, \mathbf{X}^{(k)}_v\right)$ is submodular given node $u$ with respect to $S$.
\end{lemma}

\begin{proof}
$\forall S \subseteq T, x \notin T$, denote
\begin{equation}
\begin{aligned}
y=\mathop{\arg\min}_{v \in \sigma(T\cup\{x\})}d\left(\mathbf{X}^{(k)}_u, \mathbf{X}^{(k)}_v\right).
\end{aligned}
\nonumber
\end{equation}
Since $\sigma(T\cup\{x\})=\sigma(S)\cup\sigma(\{x\})\cup\left(\sigma(T)\backslash\sigma(S)\right)$, \\
1. If $y \in \sigma(S)$, 
\begin{equation}
\begin{aligned}
    &\mathop{\min}_{v \in \sigma(S)}d\left(\mathbf{X}^{(k)}_u, \mathbf{X}^{(k)}_v\right)=\mathop{\min}_{v \in \sigma(S\cup\{x\})}d\left(\mathbf{X}^{(k)}_u, \mathbf{X}^{(k)}_v\right)\\
    =&\mathop{\min}_{v \in \sigma(T)}d\left(\mathbf{X}^{(k)}_u, \mathbf{X}^{(k)}_v\right)=\mathop{\min}_{v \in \sigma(T\cup\{x\})}d\left(\mathbf{X}^{(k)}_u, \mathbf{X}^{(k)}_v\right)\\
    =&\ d\left(\mathbf{X}^{(k)}_u, \mathbf{X}^{(k)}_y\right).
\end{aligned}
\nonumber
\end{equation}
Therefore,
\begin{equation}
\begin{aligned}
    &\mathop{\min}_{v \in \sigma(S)}d\left(\mathbf{X}^{(k)}_u, \mathbf{X}^{(k)}_v\right)-\mathop{\min}_{v \in \sigma(S\cup\{x\})}d\left(\mathbf{X}^{(k)}_u, \mathbf{X}^{(k)}_v\right)\\
    =&\mathop{\min}_{v \in \sigma(T)}d\left(\mathbf{X}^{(k)}_u, \mathbf{X}^{(k)}_v\right)-\mathop{\min}_{v \in \sigma(T\cup\{x\})}d\left(\mathbf{X}^{(k)}_u, \mathbf{X}^{(k)}_v\right).
\end{aligned}
\label{proof_situation1}
\end{equation}
\\ 
2. If $y \in \sigma(\{x\})$,
\begin{equation}
\begin{aligned}
    &\mathop{\min}_{v \in \sigma(S)}d\left(\mathbf{X}^{(k)}_u, \mathbf{X}^{(k)}_v\right)-\mathop{\min}_{v \in \sigma(S\cup\{x\})}d\left(\mathbf{X}^{(k)}_u, \mathbf{X}^{(k)}_v\right)\\
    =&\mathop{\min}_{v \in \sigma(S)}d\left(\mathbf{X}^{(k)}_u, \mathbf{X}^{(k)}_v\right)-d\left(\mathbf{X}^{(k)}_u, \mathbf{X}^{(k)}_y\right)\\
    \geq &\mathop{\min}_{v \in \sigma(T)}d\left(\mathbf{X}^{(k)}_u, \mathbf{X}^{(k)}_v\right)-d\left(\mathbf{X}^{(k)}_u, \mathbf{X}^{(k)}_y\right)\\
    =&\mathop{\min}_{v \in \sigma(T)}d\left(\mathbf{X}^{(k)}_u, \mathbf{X}^{(k)}_v\right)-\mathop{\min}_{v \in \sigma(T\cup\{x\})}d\left(\mathbf{X}^{(k)}_u, \mathbf{X}^{(k)}_v\right).
\end{aligned}
\label{proof_situation2}
\end{equation}
The third line of Eq. \eqref{proof_situation2} is due to Lemma \ref{lemma:nondecreasing}. \\ 
3. If $y \in \sigma(T) \backslash \sigma(S)$,
\begin{equation}
\begin{aligned}
    \mathop{\min}_{v \in \sigma(T)}d\left(\mathbf{X}^{(k)}_u, \mathbf{X}^{(k)}_v\right)=\mathop{\min}_{v \in \sigma(T\cup\{x\})}d\left(\mathbf{X}^{(k)}_u, \mathbf{X}^{(k)}_v\right)=\ d\left(\mathbf{X}^{(k)}_u, \mathbf{X}^{(k)}_y\right).
\end{aligned}
\nonumber
\end{equation}
Since 
\begin{equation}
\begin{aligned}
\mathop{\min}_{v \in \sigma(S)} d\left(\mathbf{X}^{(k)}_u, \mathbf{X}^{(k)}_v\right)-\mathop{\min}_{v \in \sigma(S\cup\{x\})}d\left(\mathbf{X}^{(k)}_u, \mathbf{X}^{(k)}_v\right) \geq\ 0
\end{aligned}
\nonumber
\end{equation}
due to Lemma \ref{lemma:nondecreasing}, we have,
\begin{equation}
\begin{aligned}
    &\mathop{\min}_{v \in \sigma(S)}d\left(\mathbf{X}^{(k)}_u, \mathbf{X}^{(k)}_v\right)-\mathop{\min}_{v \in \sigma(S\cup\{x\})}d\left(\mathbf{X}^{(k)}_u, \mathbf{X}^{(k)}_v\right)\\
    \geq&\mathop{\min}_{v \in \sigma(T)}d\left(\mathbf{X}^{(k)}_u, \mathbf{X}^{(k)}_v\right)-\mathop{\min}_{v \in \sigma(T\cup\{x\})}d\left(\mathbf{X}^{(k)}_u, \mathbf{X}^{(k)}_v\right).
\end{aligned}
\label{proof_situation3}
\end{equation}
\end{proof}

Now we prove Theorem \ref{theorem}:\\ 
1. $\forall S \subseteq T, x \notin T$,
\begin{equation}
\begin{aligned}
    &D_{NN}(S\cup\{x\})-D_{NN}(S)\\
    =&\sum_{u \in V}\mathop{\min}_{v \in \sigma(S)}d\left(\mathbf{X}^{(k)}_u, \mathbf{X}^{(k)}_v\right)-\mathop{\min}_{v \in \sigma(S\cup\{x\})}d\left(\mathbf{X}^{(k)}_u, \mathbf{X}^{(k)}_v\right)\\
    \geq&\sum_{u \in V}\mathop{\min}_{v \in \sigma(T)}d\left(\mathbf{X}^{(k)}_u, \mathbf{X}^{(k)}_v\right)-\mathop{\min}_{v \in \sigma(T\cup\{x\})}d\left(\mathbf{X}^{(k)}_u, \mathbf{X}^{(k)}_v\right)\\
    =&\ D_{NN}(T\cup\{x\})-D_{NN}(T).
\end{aligned}
\label{eq:submodular}
\end{equation}
The third line of Eq. \eqref{eq:submodular} is due to Lemma \ref{lemma:submodular}. Therefore, $D_{NN}(S)$ is submodular with respect to S. \\
2. $\forall S \subseteq T$,
\begin{equation}
\begin{aligned}
    &D_{NN}(T)-D_{NN}(S)\\
    =&\sum_{u \in V}\mathop{\min}_{v \in \sigma(S)}d\left(\mathbf{X}^{(k)}_u, \mathbf{X}^{(k)}_v\right)-\mathop{\min}_{v \in \sigma(T)}d\left(\mathbf{X}^{(k)}_u, \mathbf{X}^{(k)}_v\right)\\
    \geq&\ 0.
\end{aligned}
\label{eq:monotone}
\end{equation}
The last line of Eq. \eqref{eq:monotone} is due to Lemma \ref{lemma:nondecreasing}. Therefore, $D_{NN}(S)$ is nondecreasing with respect to S.

\end{proof}

\para{A.2\ \  Datasets description}
\begin{table*}[t]
\small
\centering
\caption{Overview of the Four Datasets} \label{Dataset}
\begin{tabular}{ccccccccc}
\toprule
\textbf{Dataset}&\textbf{\#Nodes}& \textbf{\#Features}&\textbf{\#Edges}&\textbf{\#Classes}&\textbf{\#Train/Val/Test}&\textbf{Task type}&\textbf{Description}\\
\midrule
Cora& 2,708 & 1,433 &5,429&7& 1208/500/1000 & Transductive&citation network\\
Citeseer& 3,327 & 3,703&4,732&6& 1827/500/1000 & Transductive&citation network\\
Pubmed& 19,717 & 500 &44,338&3& 18217/500/1000 & Transductive&citation network\\
ogbn-papers100M & 111,059,956 & 128 & 1,615,685,872 & 172 &
110M/125K/214K&Transductive &citation network \\
\midrule
Reddit& 232,965 & 602 & 11,606,919 & 41 &  155K/23K/54K & Inductive&social network \\
\bottomrule
\end{tabular}
\end{table*}

\textbf{Cora}, \textbf{Citeseer}, and \textbf{Pubmed}\footnote{https://github.com/tkipf/gcn/tree/master/gcn/data} are three popular citation network datasets, and we follow the public training/validation/test split in GCN ~\cite{DBLP:conf/iclr/KipfW17}.
In these three networks, papers from different topics are considered as nodes and the edges are citations among the papers.  The node attributes are binary word vectors and class labels are the topics papers belong to.

\noindent\textbf{Reddit} is a social network dataset derived from the community structure of numerous Reddit posts. It is a well-known inductive training dataset and the training/validation/test split in our experiment is the same as the split in GraphSAGE~\cite{hamilton2017inductive}. The public version of Reddit and Flickr provided by GraphSAINT\footnote{https://github.com/GraphSAINT/GraphSAINT} is used in our paper.

\noindent\noindent\textbf{ogbn-papers100M} is a paper citation dataset with 111 million papers indexed by MAG~\cite{wang2020microsoft} in it. This dataset is known as the largest existing public node classification dataset currently and is much larger than others. We follow the official training/validation/test split and metric released in official website\footnote{https://github.com/snap-stanford/ogb} and official paper~\cite{hu2020ogb}.

\para{A.3\ \  More details about the compared GNNs}

The main characteristic of all baselines are listed bellow:
\begin{itemize}
    \item \textbf{GCN}~\cite{DBLP:conf/iclr/KipfW17} produces node embedding vectors by truncating the Chebyshev polynomial to the first-order neighborhoods.
    
    \item \textbf{APPNP}~\cite{klicpera2018predict} uses the relationship between GCN and PageRank to derive an improved propagation scheme based on personalized PageRank.
    
    \item \textbf{SGC}~\cite{wu2019simplifying} reduces the excess complexity of GCN through successively removing nonlinearities and collapsing weight matrices between consecutive layers.
    
    \item \textbf{MVGRL}~\cite{hassani2020contrastive} is a self-supervised approach for learning node and graph level representations by contrasting structural views of graphs.
\end{itemize}

\para{A.4\ \ Concrete settings}

To ensure impartiality, we use the 2-layer GCN with a hidden size of 128 on all datasets. The dropout rate is 0.85 and the L2 regularization is $5*10^{-4}$ for all datasets. For \sys(ball-D), the threshold $\theta$ in is 0.25 and the radius $r$ 0.05 for all datasets. For \sys(NN-D), the $\gamma$ for all datasets is set to 1. Besides, $\alpha$ is set as 0.1 in APPNP. For MVGRL, we use the official settings.
To obtain a well-trained model and ensure the reliability of the model-based selection criteria in AGE and ANRMAB, 200 training epochs are allocated in each iteration. 
In addition, we choose $k$ (the number of classes) nodes to label in each iteration, i.e., $k$ in Cora is 7.
Notably, for AGE and ANRMAB, the initial nodes are randomly selected, thus achieving the same accuracy as the Random method at the beginning.

We implement AGE following its open-sourced implementation and ANRMAB based on its original paper. The experiments are conducted on a machine with Ubuntu 16.04,48 Intel(R) Xeon(R) CPUs (E5-2650 v4 @ 2.20GHz), and four NVIDIA GeForce GTX 1080 Ti GPUs. The code is written using Python 3.6, Pytorch 1.7.1~\cite{paszke2019pytorch} and CUDA 10.1. Notably, in the AL setting, we assume that the label given by the oracle is always correct. This assumption is reasonable since handling noisy oracle is orthogonal with our work and our main contributions are the node selection criterion. In practice, either experts or crowds like MTurk~\cite{sorokin2008utility} are allowed to provide the label, and one could introduce an existing label noisy detection or correction techniques to work with \sys.

The detailed package requirements for \sys are listed on \url{https://github.com/zwt233/Grain}. Please check the ``requirements.txt'' in the root directory. % For the sake of stability, the recommended version of \texttt{scikit-learn} is 0.21.3.
To reproduce the end-to-end results of \sys, run the notebook file  \texttt{test.ipynb} under directory `examples'.

\para{A.5\ \  More details about the other node selection methods.}

We compare \sys with the following baselines: 
\begin{itemize}
    \item \textbf{Random}: Select nodes randomly;
    \item \textbf{Degree}: Select nodes with maximum degree;
    \item \textbf{AGE}~\cite{cai2017active}: Combine different query strategies linearly with time-sensitive parameters for GNNs;
    \item \textbf{ANRMAB}~\cite{gao2018active}: Adopt a multi-armed bandit mechanism for adaptive decision making to select nodes for GNNs; 
    \item \textbf{K-Center-Greedy} (KCG)~\cite{sener2018active}: Select nodes that maximize the distance from the nearest clustering center;
    \item \textbf{\sys(ball-D)}: Use balls in semantic space to control the influence magnitude and diversity with the objective~(13);
    \item \textbf{\sys(NN-D)}: Use the dynamic parameter $\lambda$ to combine influence magnitude and diversity with the objective~(11).
\end{itemize}
Both AGE and ANRMAB are designed for active learning on GNNs, which adopt the uncertainty, density, and node degree to select nodes. 
They have to train the evaluated model since both the uncertainty and density are calculated based on the model predictions. 
By contrast, \sys considers this problem from the perspective of diversified influence maximization and the selection process is model-free.

 \begin{figure*}[tp]
\centering  
\subfigure[Cora]{
\label{Fig.single}
\includegraphics[width=0.28\textwidth]{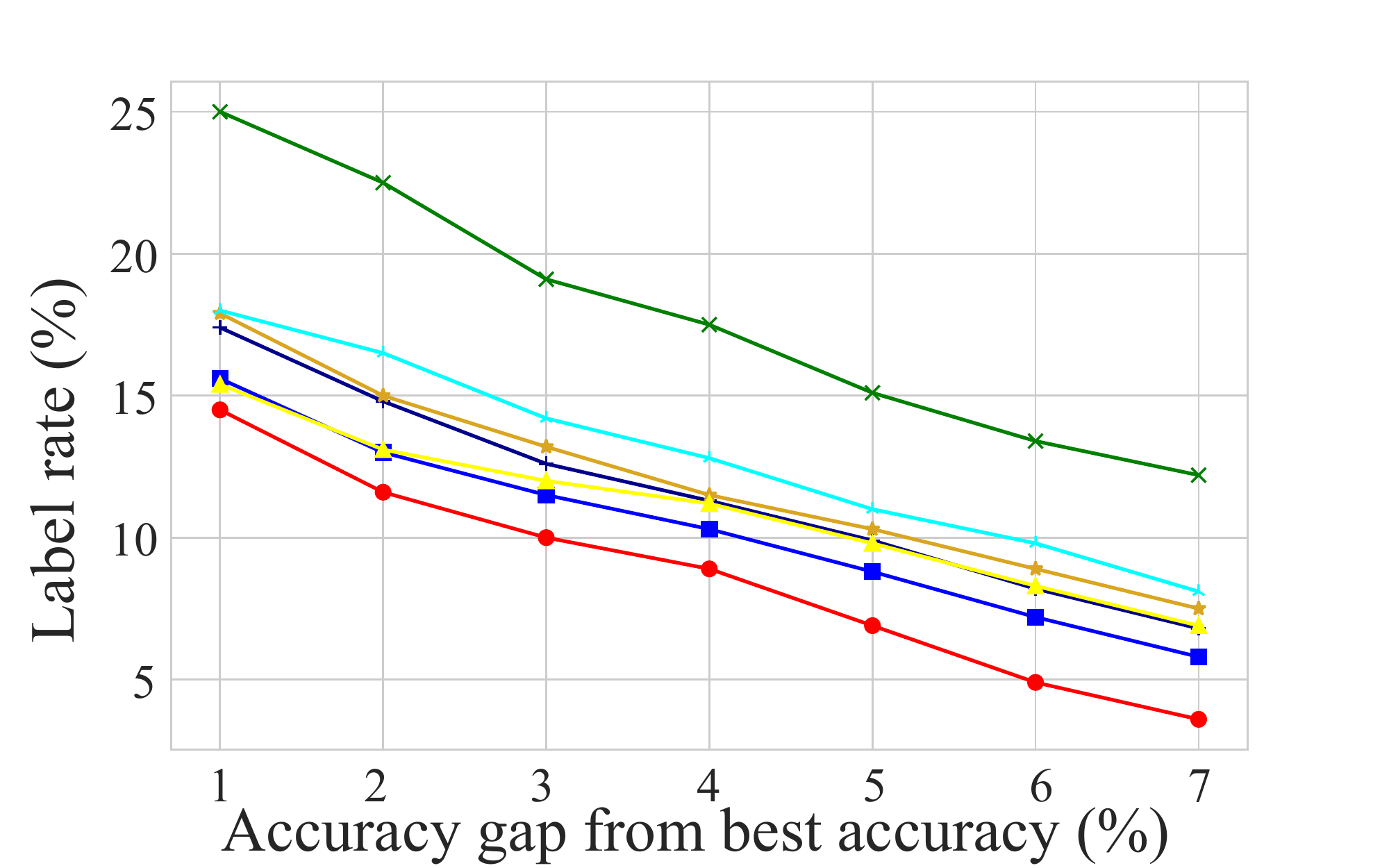}}
\subfigure[Citeseer]{
\label{Fig.ensemble}
\includegraphics[width=0.28\textwidth]{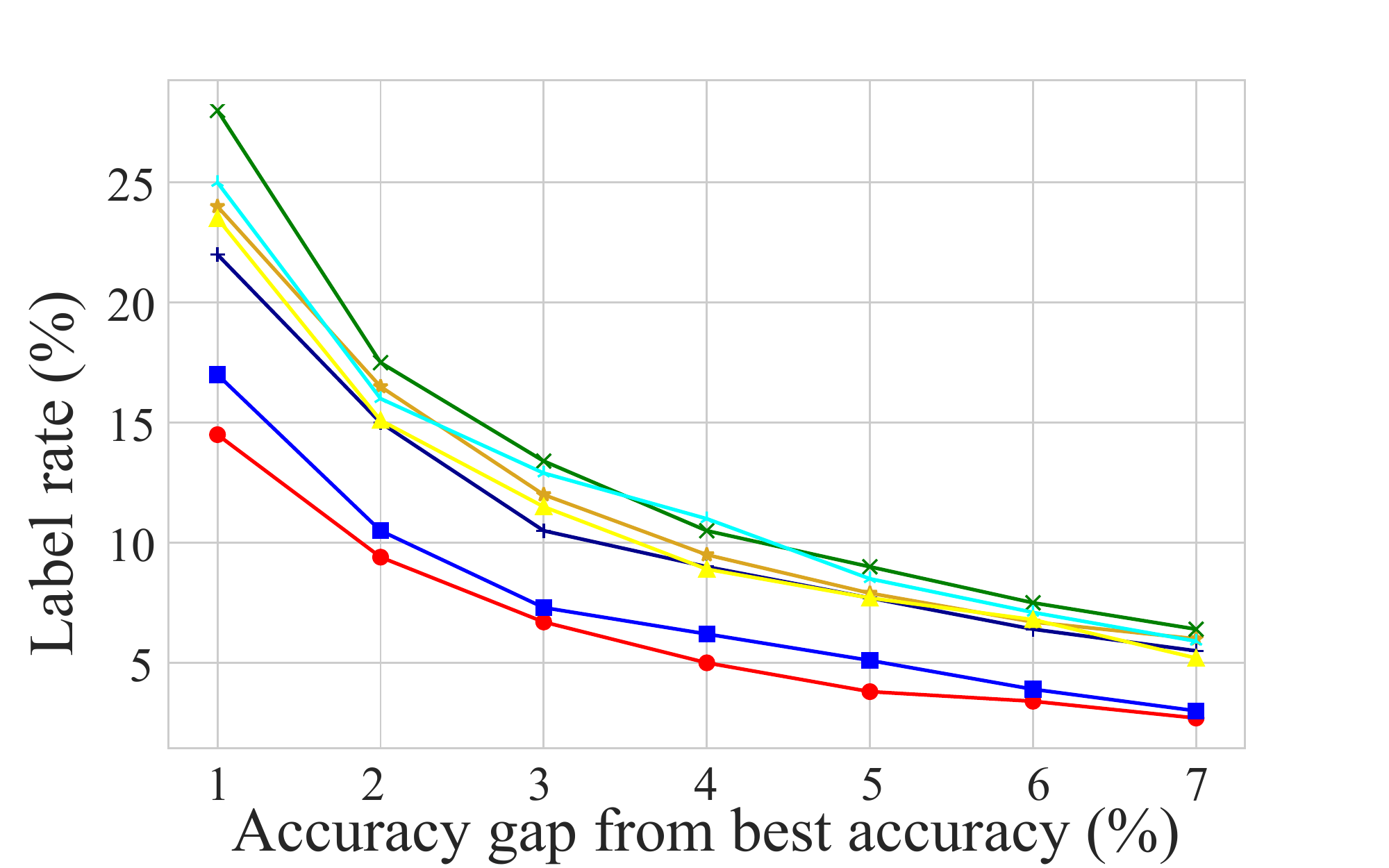}}
\subfigure[PubMed]{
\label{Fig.ensemble}
\includegraphics[width=0.28\textwidth]{figure/performance_coreset_PubMed_result.pdf}}
\subfigure{
\label{Fig.ensemble}
\includegraphics[width=0.08\textwidth]{figure/legend1.pdf}}
\vspace{-5.5mm}
\caption{The label rate (the percentage of the whole available node set) to reach a accuracy gap from the best accuracy.}
\label{fig.coreset_performance}
\vspace{-4mm}
\end{figure*}

\begin{figure}[tp]
% \vspace{-4mm}
\centering   
\subfigure{
\label{fig.effi_nogpu1}
\includegraphics[width=0.252\textwidth]{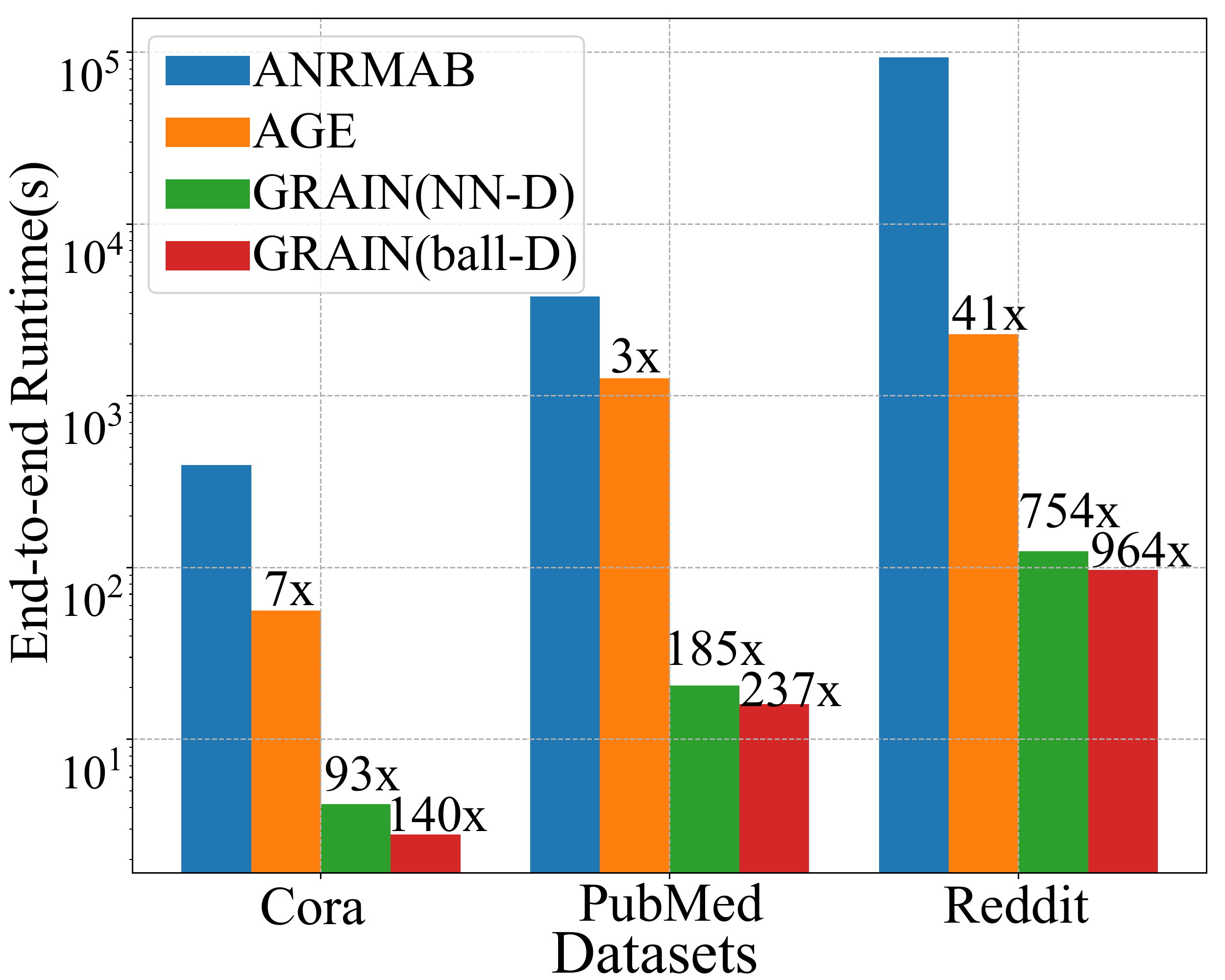}}
\subfigure{
\label{fig.effi_nogpu2}
\includegraphics[width=0.208\textwidth]{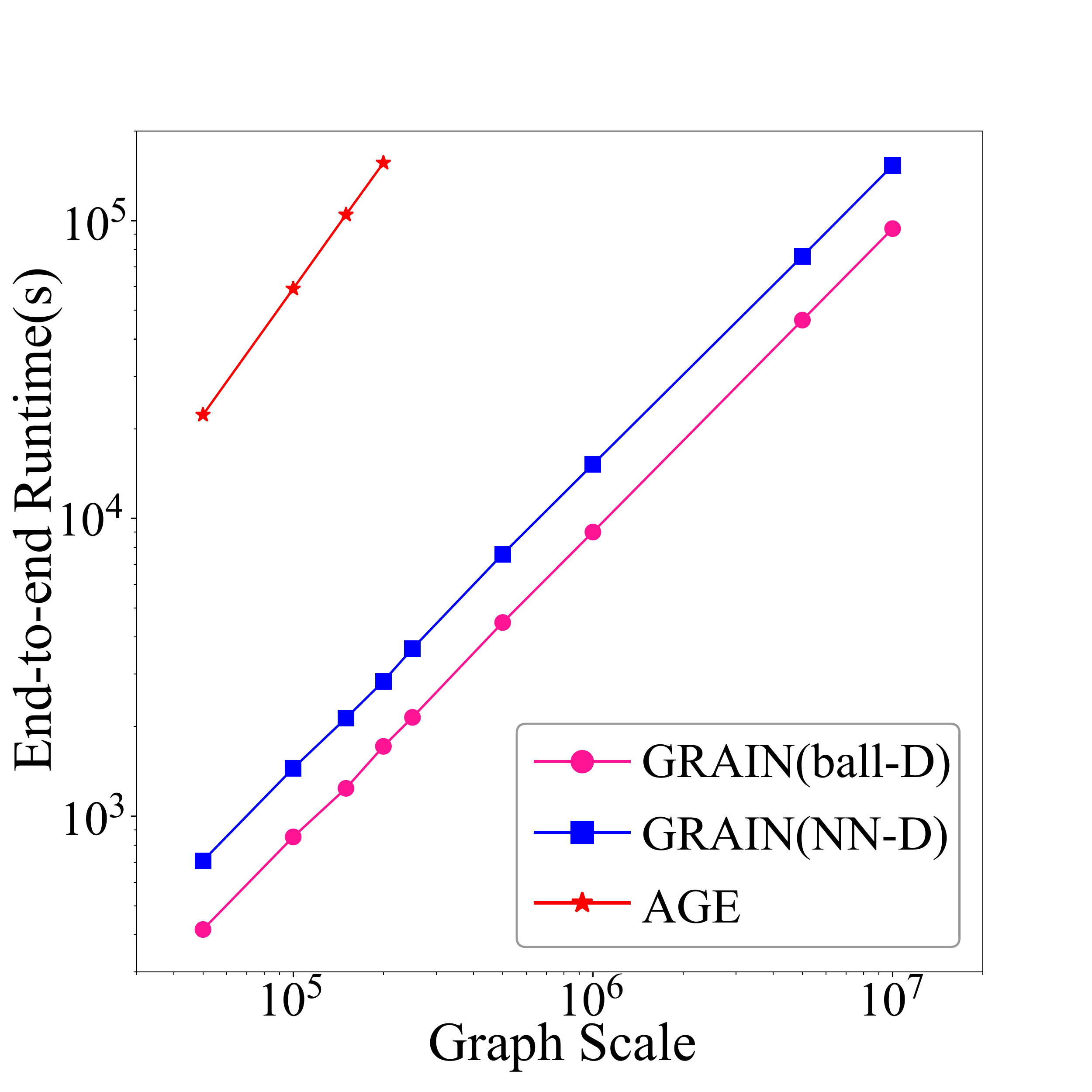}}
\vspace{-5.5mm}
\caption{End-to-end runtime (at log scale) using CPU on Cora, Citeseer, PubMed and ogbn-papers100M with different graph scales (at log scale). The speedup on each bar is relative to ANRMAB.}
\label{speed}
% \vspace{-4mm}
\end{figure}
\para{A.6 Efficiency comparison on CPU}

As \sys is model-free and does not rely on GPU, we also evaluate it on CPU environment and the result is shown in Figure~2.
 Obviously,Figure~2 shows that the speedup grows on all these three datasets. For example on Citeseer, compared to ANRMAB, the speedup of \sys(ball-D) increases from 44× to 237× when the GPU device is unavailable.
 For the 1M graph scale on ogbn-papers100M, the speedup increases from 20× to 69× when we switch from GPU to CPU.
In addition, both \sys(ball-D) and \sys(NN-D) achieve higher speedups on larger graphs. Compared with AGE on ogbn-papers100M, the speedup of \sys(ball-D) increases from 54× to 91× on 500k and 2M on CPU, respectively. 
 Without GPU,  AGE is unable to scale to 2.5M graph scale within two weeks in ogbn-papers100M, while our proposed \sys is model-free and can scale to large graphs easily.

\para{A.7 Core-set selection on other graphs}

Core-set selection starts with a large labeled or unlabeled dataset and aim to find a small subset that accurately approximates the full dataset.
We first train the exact model with the full training labels (i.e., 1208 for Cora and 18217 for PubMed), and get the best test accuracy, which is 86.3\% on Cora, 77.2\% on Citeseer, and 86.5\% on PuMed.
As shown in Figure~\ref{fig.coreset_performance}, we set a gap between exact and approximate models, which ranges from 1\% to 7\%, and then evaluate the number of labeled nodes each core-set selection method requires to achieve the corresponding accuracy gap. 
It is obvious that both \sys(ball-D) and \sys(NN-D) significantly outperform the other baselines and achieve the same accuracy gap using much fewer labeled nodes.
Concretely, to achieve an accuracy gap of 2\% on PubMed, AGE needs 3.2\% of all labeled training nodes while \sys(ball-D) needs only 1\% of these nodes, which means \sys(ball-D) outperform AGE by 3.2$\times$ in terms of data efficiency.

\bibliographystyle{ACM-Reference-Format}
\balance
\bibliography{reference}

\end{document}